\documentclass[hidelinks,onefignum,onetabnum]{siamart220329}

\usepackage{lipsum}
\usepackage{amsfonts}
\usepackage{amssymb}
\usepackage{graphicx}
\usepackage{epstopdf}
\usepackage{bbm}
\usepackage{algorithmic}

\ifpdf
  \DeclareGraphicsExtensions{.eps,.pdf,.png,.jpg}
\else
  \DeclareGraphicsExtensions{.eps}
\fi

\newsiamremark{remark}{Remark}
\newsiamremark{hypothesis}{Hypothesis}
\crefname{hypothesis}{Hypothesis}{Hypotheses}
\newsiamthm{claim}{Claim}
\newsiamthm{assumption}{Assumption}
\newcommand{\rev}[1]{{\color{black} #1}}

\headers{Entropy-Regularized NPG with Function Approximation}{Semih Cayci, Niao He, R. Srikant}

\title{Convergence of Entropy-Regularized Natural Policy Gradient with Linear Function Approximation}

\author{Semih Cayci\thanks{Chair of Mathematics of Information Processing, RWTH Aachen University, Aachen, Germany (\email{cayci@mathc.rwth-acchen.de}).}
\and Niao He\thanks{Department of Computer Science, ETH Zurich, Zurich, Switzerland 
  (\email{niao.he@inf.ethz.ch}).}
\and R. Srikant\thanks{Department of Electrical and Computer Engineering, Coordinated Science Laboratory, University of Illinois at Urbana-Champaign, Urbana, IL, USA (\email{rsrikant@illinois.edu})}.}

\usepackage{amsopn}

\newcommand\blfootnote[1]{%
  \begingroup
  \renewcommand\thefootnote{}\footnote{#1}%
  \addtocounter{footnote}{-1}%
  \endgroup
}

\newcommand{\cS}{\mathbb{S}}
\newcommand{\cA}{\mathbb{A}}
\newcommand{\cP}{\mathcal{P}}
\newcommand{\cQ}{\mathcal{Q}}
\newcommand{\cV}{\mathcal{V
}}

\newcommand{\bE}{\mathbb{E}}
\newcommand{\bP}{\mathbb{P}}
\newcommand{\bR}{\mathbb{R}}
\newcommand{\ea}{\epsilon_{approx}}
\newcommand{\dopt}{d_\mu^{\pi^*}(s)\pi^*(a|s)}
\newcommand{\pit}{{\pi_{\theta}}}
\newcommand{\pitp}{{\pi_{\theta^\prime}}}

\ifpdf
\hypersetup{
  pdftitle={Convergence of Entropy-Regularized Natural Policy Gradient with Linear Function Approximation},
  pdfauthor={Semih Cayci, Niao He, and R. Srikant}
}
\fi

\begin{document}

\maketitle
\blfootnote{\textbf{Funding:} This work was supported by NSF Grants CCF 22-07547, CCF 19-34986, CNS 21-06801, ONR Grant N00014-19-1-2566, and SNSF Project Funding No. 200021-207343.}
\begin{abstract}
Natural policy gradient (NPG) methods, equipped with function approximation and entropy regularization, achieve impressive empirical success in reinforcement learning problems with large state-action spaces. However, their convergence properties and the impact of entropy regularization remain elusive in the function approximation regime. In this paper, we establish  finite-time convergence analyses of entropy-regularized NPG with linear function approximation under softmax parameterization. In particular, we prove that entropy-regularized NPG with averaging satisfies the \emph{persistence of excitation} condition, and achieves a fast convergence rate of $\tilde{O}(1/T)$ up to a function approximation error in regularized Markov decision processes. This convergence result does not require any a priori assumptions on the policies. Furthermore, under mild regularity conditions on the concentrability coefficient and basis vectors, we prove that entropy-regularized NPG exhibits \emph{linear convergence} up to the compatible function approximation error. Finally, we provide sample complexity results for sample-based NPG with entropy regularization.
\end{abstract}

\begin{keywords}
reinforcement learning, policy gradient, natural policy gradient
\end{keywords}

\section{Introduction}
The goal of reinforcement learning (RL) is to sequentially maximize the expected total reward in a Markov decision process (MDP) \cite{sutton2018reinforcement, szepesvari2010algorithms, bertsekas1996neuro}. Policy gradient (PG) methods directly find the optimal policy in the parameter space by using gradient ascent \cite{williams1992simple, sutton1999policy, konda2000actor}, and they have demonstrated remarkable success in a broad class of challenging reinforcement learning problems such as chess, Go, healthcare applications, networking and robotics. The success largely benefits from  the versatility of PG methods in accommodating a rich class of parameterization and function approximation schemes \cite{mnih2016asynchronous, silver2016mastering, nachum2017trust, duan2016benchmarking}.

Among the variants of PG methods, natural policy gradient (NPG)~\cite{kakade2001natural,peters2008natural}, has been particularly popular. NPG uses Fisher information matrix for pre-conditioning the gradient steps and resembles a quasi-Newton method \cite{amari1998natural}. The idea of natural policy gradient has also been widely explored and generalized in many other RL algorithms~\cite{bhatnagar2007incremental,schulman2015trust,schulman2017proximal}.

Besides function approximation, the success of policy gradient methods has also been attributed to the use of \emph{entropy regularization}, a common algorithmic technique to encourage exploration of learning policies~\cite{haarnoja2018soft, nachum2017trust}.  
The impact of entropy regularization for policy gradient methods has been extensively studied recently, from both empirical and theoretical perspectives; see, e.g.~\cite{ahmed2019understanding,neu2017unified,geist2019theory,agarwal2020optimality,cen2020fast,mei2020global,lan2021policy}, just to name a few. However, existing results for the most part only studied the convergence properties of entropy regularized policy gradient methods in the \emph{tabular setting}, leaving a wide gap between the theory and the practice. 

In this paper, we aim to shed light on the theoretical effectiveness of the entropy regularization on policy gradient methods in the critical \emph{function approximation regime}. As a key stepping stone, we will focus on \emph{NPG with log-linear policy class}.

To the best of our knowledge, this is the first work that analyzes the convergence of NPG with entropy regularization in the function approximation regime.  Extensions to other PG methods and nonlinear function approximation are out of the scope of this work.

\subsection{Main Contributions}
In this work, we establish sharp non-asymptotic convergence bounds for entropy-regularized natural policy gradient under softmax parameterization with linear function approximation, and elucidate the theoretical benefits of entropy regularization in policy optimization.

Our main contributions include the following:
\begin{itemize}
\setlength\itemsep{0em}

    \item \textit{Fast $\tilde{O}(1/T)$ convergence of entropy-regularized NPG under a weak regularity condition:} We show that the persistence of excitation condition (i.e., all actions are explored with some probability bounded away from zero), is satisfied under entropy regularization. This condition ensures sufficient exploration, and consequently we prove that entropy-regularized NPG with averaging and gradient clipping achieves a $\tilde{O}(1/T)$ convergence rate up to a function approximation error under a very weak regularity condition on the concentrability coefficient.
    
    \item \textit{Linear convergence of entropy-regularized NPG with function approximation:} We prove that entropy-regularized NPG under softmax parameterization and linear function approximation achieves a much faster linear convergence rate $\exp(-\Omega(T))$ up to a function approximation error under additional but mild regularity conditions on the basis vectors and concentrability coefficient.
\end{itemize}

Finally, building on the analysis, we further characterize the convergence of NPG when the natural gradient can only be estimated from samples and computed inexactly. In particular, we extend our results to natural actor-critic methods based on entropy-regularized NPG for actor update and temporal difference learning for critic update. 

\subsection{Related Work}
\textit{NPG with function approximation:} In \cite{agarwal2020optimality}, unregularized NPG with softmax parameterization and linear function approximation was studied, and $O(1/\sqrt{T})$ convergence rate up to approximation and statistical errors is proved. Our analysis is inspired by \cite{agarwal2020optimality} to analyze entropy-regularized NPG. We prove that linear convergence rate is achieved under mild regularity conditions on the basis vectors and concentrability coefficient. In \cite{wang2019neural} and \cite{xu2020improving}, sublinear convergence rates for NPG with neural network approximation, and general (smooth) function approximation are proved, respectively, under similar concentrability assumptions. {In recent works \cite{alfano2022linear, yuan2022linear}, linear convergence of NPG with log-linear approximation was established without regularization. These results rely on the assumption of a bounded relative condition number, which stems from a using good initial state-action distribution for sampling, and appears in the error upper bound. Although the existence of such a good initial state-action distribution for sampling is proven in \cite{agarwal2020optimality}, no explicit construction is known. As we prove in this paper, under entropy regularization, such a strong assumption is not required because of sufficient exploration induced by regularization (see Lemma \ref{lemma:q-npg}), and one can achieve near-optimality under minimal conditions at the expense of a regularization bias.}

\textit{NPG/PG in the tabular setting:} The convergence properties of NPG in the tabular setting are relatively better understood compared to the function approximation setting \cite{agarwal2020optimality, bhandari2019global, cen2020fast, mei2020global, khodadadian2021linear}. In \cite{shani2020adaptive}, adaptive TRPO with decaying step-size was shown to achieve $O(1/T)$ convergence rate in the tabular setting. In \cite{cen2020fast}, linear convergence of tabular-NPG is proved by exploiting a relation to the policy iteration in the tabular setting. In another recent work, \cite{mei2020global} proves linear convergence of entropy-regularized tabular-PG by establishing a Polyak-{\L}ojasiewicz inequality. Similar results are obtained in~\cite{lan2021policy} for more general regularizers. The function approximation regime is fundamentally different than the tabular setting, and it is important to establish fast convergence rates in this regime \cite{cen2020fast}. In this work, we adopt a Lyapunov-drift approach, which makes use of potential functions, to prove fast convergence rates in the function approximation regime. This Lyapunov approach also enables us to study and explain the effectiveness of entropy regularization directly.

\subsection{Notation}
For a finite set $\cS$, we denote its cardinality as $|\cS|$. For a matrix $A\in\bR^{d\times d}$, we denote the {singular values} of $A$ in ascending order by $\sigma_1(A)\leq \sigma_2(A)\leq \ldots \leq \sigma_d(A)$. For two distributions $P, Q$, we denote $P\ll Q$ if $Q(A)=0$ implies $P(A)=0$ for any event $A$. We denote the Kullback-Leibler \rev{and $\chi^2$} divergences between any $P, Q$ as 
\begin{align*}
    \mathcal{D}_{\mathsf{KL}}(P\|Q) &= \bE_{x\sim P}\left[\log\frac{P(x)}{Q(x)}\right],\\
\chi^2(P\|Q) &= \bE_{x\sim Q}\left[\frac{(Q(x)-P(x))^2}{Q^2(x)}\right],
\end{align*}
respectively. The uniform distribution over a finite set $B$ is denoted as $\mathsf{Unif}(B)$.

\section{System Model and Algorithms}
In this section, we will introduce the reinforcement learning setting and natural policy gradient algorithm.

\subsection{Markov Decision Processes}
In this work, we consider a discounted Markov decision process $(\cS, \cA, \cP, r, \gamma)$, where $\cS$ and $\cA$ are the state and action spaces,  $\cP$ is a transition model, $r(s,a)\in[0,r_{max}],~(s,a)\in\cS\times\cA$ for some $r_{max}<\infty$ is the reward function, and $\gamma\in(0,1)$ is the discount factor. Specifically, upon taking an action $a\in\cA$ at state $\cS$, the controller receives a reward $r(s,a)$, and the system makes a transition into a state $s^\prime\in\cS$ with probability $\cP(s^\prime|s,a)$. In this work, we consider a finite but arbitrarily large state space $\cS$ for simplicity, and a finite action space $\cA$.

A stationary randomized \textit{policy} $\pi$ corresponds to a decision-making rule by specifying the probability $\pi(a|s)$ of taking an action $a\in\cA$ at a given state $s\in\cS$. A policy $\pi$ introduces a trajectory by specifying $a_t \sim \pi(\cdot|s_t)$ and $s_{t+1}\sim \cP(\cdot |s_t,a_t)$ given an initial state $s_0=s\in\cS$. The corresponding value function of a policy $\pi$ is as follows:
\begin{equation}
    \mathcal{V}^\pi(s) = \bE\Big[ \sum_{t=0}^\infty \gamma^t r(s_t,a_t) | s_0 = s\Big],
\end{equation}
where $a_t\sim\pi(\cdot|s_t)$ and $s_{t+1}\sim\cP(\cdot|s_t,a_t)$. For an initial state distribution $\mu$, we define (with a slight abuse of notation)
\begin{equation}
    \mathcal{V}^\pi(\mu) = \sum_{s\in\cS}\mu(s)\mathcal{V}^\pi(s).
\end{equation}
\textbf{Policy parameterization:} We consider softmax parameterization with linear function approximation. Namely, we consider the log-linear policy class $\Pi= \{\pi_\theta:\theta\in\bR^d\}$, where:
\begin{equation}
    \pi_\theta(a|s) = \frac{\exp(\theta^\top \phi_{s,a})}{\sum_{a^\prime\in\cA}\exp(\theta^\top \phi_{s,a^\prime})},
    \label{eqn:policy-class}
\end{equation}
for a set of $d$-dimensional basis vectors $\{\phi_{s,a}\in\bR^d:s\in\cS,a\in\cA\}$ with $\|\phi_{s,a}\|_2 \leq 1$ for all $(s,a)\in\cS\times\cA$, and policy parameter $\theta\in\bR^d$. Note that the policy class $\Pi$ is a restricted policy class, which is a strict subset of all stochastic policies \cite{agarwal2020optimality}.

\textbf{Entropy regularization:} The value function $\mathcal{V}^{\pi_\theta}(\mu)$ is a non-concave function of $\theta\in\bR^d$, and there exist suboptimal near-deterministic policies. In order to encourage exploration and evade suboptimal near-deterministic policies, entropy regularization is commonly used in practice \cite{silver2016mastering, haarnoja2018soft, nachum2017trust,ahmed2019understanding}. For a policy $\pi\in\Pi$, let
\begin{equation}
    \mathcal{H}^{\pi}(\mu) = \bE\Big[ \sum_{t=0}^\infty \gamma^t h\big(\pi(\cdot|s_t)\big)\Big|s_0\sim\mu \Big],
\end{equation}
where $h(\pi(\cdot|s)) = -\sum_{a\in\cA}\pi(a|s)\log\big(\pi(a|s)\big)$ is the entropy functional. Then, for $\lambda > 0$, the entropy-regularized value function is defined as follows:
\begin{equation}
    \mathcal{V}_\lambda^\pi(\mu) = \mathcal{V}^\pi(\mu) + \lambda \mathcal{H}^{\pi}(\mu).
    \label{equation:H-reg-vf}
\end{equation}
Note that the maximum-entropy policy, $\pi_{\mathbf{0}}(a|s) = 1/|\cA|$ for all $(s,a)\in\cS\times\cA$, maximizes the regularizer $\mathcal{H}^\pi(\mu)$. Hence, the additional $\lambda \mathcal{H}^\pi(\mu)$ term in \eqref{equation:H-reg-vf} encourages exploration increasingly with $\lambda > 0$. Since $\mathcal{H}^\pi(\mu)\leq \frac{\log|\cA|}{1-\gamma}$, $\lambda \leq \frac{(1-\gamma)\epsilon}{\log|\cA|}$ implies $\epsilon$-proximity to the unregularized value function \cite{cen2020fast}.

\textbf{Objective:} Our goal is to maximize the entropy-regularized value function in \eqref{equation:H-reg-vf} for a given $\lambda > 0$ and initial state distribution $\mu$:
\begin{equation}
    \theta^* = \arg\max_{\theta\in\bR^d} \mathcal{V}_\lambda^{\pi_\theta}(\mu).
\end{equation}
We denote the optimal policy as $\pi^* = \pi_{\theta^*}$ throughout the paper, and assume that $\|\theta^*\|_2 < \infty$, which automatically holds for sufficiently large $\lambda > 0$.

\textbf{Soft Q-function: }We define the soft Q-function and shifted Q-function under a policy $\pi$ as follows, respectively:
\begin{align}
    \mathcal{Q}_\lambda^\pi(s, a) &= r(s,a) + \gamma \sum_{s^\prime\in\cS}\cP(s^\prime|s,a)\mathcal{V}_\lambda^\pi(s^\prime),\\
    \label{eqn:qfunction} q_\lambda^\pi(s, a) &= \mathcal{Q}_\lambda^\pi(s, a) - \lambda\log\pi(a|s).
\end{align}
We have the following characterization of $\mathcal{V}_\lambda^\pi(\mu)$:
\begin{align*}
    \mathcal{V}_\lambda^\pi(\mu) = \sum_{s\in\cS,a\in\cA}\mu(s)\pi(a|s)\Big(\mathcal{Q}_\lambda^\pi(s,a)-\lambda\log\pi(a|s)\Big).
\end{align*}

We can bound the entropy-regularized value function as follows: 
\begin{equation}0 \leq \mathcal{V}_\lambda^{\pi_\theta}(\mu) \leq \frac{r_{max}+\lambda\log|\cA|}{1-\gamma},
\label{eqn:val-bounds}
\end{equation}
for any $\lambda > 0$ since $r\in[0,r_{max}]$ and ${H}(P)\leq \log|\cA|$ for any distribution $P$ over $\cA$.

\subsection{Policy Gradient Theorem and Compatible Function Approximation}
In order to define NPG with function approximation, it is useful to first characterize the policy gradient with respect to the parameter $\theta$.
For the initial state distribution $\mu$, we define the state visitation distribution as $$d_{\mu}^\pi(s) = (1-\gamma)\sum_{t=0}^\infty\gamma^t\bP_\pi(s_t = s|s_0\sim \mu).$$
We also define 
$$(d_\mu^\pi\circ\pi)(s,a) = d_\mu^\pi(s)\cdot \pi(a|s),$$ as the state-action visitation distribution under a policy $\pi$.

The following proposition characterizes the gradient of the entropy-regularized value function with respect to $\theta$. This is a direct extension of the policy gradient theorem to entropy-regularized value functions with linear function approximation \cite{sutton2018reinforcement, agarwal2020optimality}.
\begin{proposition}[Policy gradient]
For any $\theta\in\bR^d$, $\lambda > 0$ and initial state distribution $\mu$, we have:
\begin{equation}
    \nabla_\theta \mathcal{V}_\lambda^{\pi_\theta}(\mu) = \frac{1}{1-\gamma}\bE\Big[\nabla_\theta\log\pi_\theta(a|s)q_\lambda^{\pi_\theta}(s,a)\Big],
\end{equation}
where the expectation is taken over $s\sim d_\mu^{\pi_\theta},a\sim\pi_\theta(\cdot|s)$ and \begin{equation}\nabla_\theta\log\pi_\theta(a|s) = \phi_{s,a}-\sum_{a^\prime\in\cA}\pi_\theta(a^\prime|s)\phi_{s,a^\prime}.
\label{eqn:grad-log}
\end{equation}
\label{thm:policy-gradient}
\end{proposition}
By using Proposition \ref{thm:policy-gradient}, 
the NPG update can be computed by the following lemma, which is an extension of \cite{kakade2001natural, agarwal2020optimality}.
\begin{lemma}[Compatible function approximation]
Let \begin{equation}
    L(w, \theta) = \bE\Big[\Big(\nabla_\theta^\top\log\pi_\theta(a|s)w-q_\lambda^{\pi_\theta}(s,a)\Big)^2\Big],
    \label{eqn:loss}
\end{equation}
be the approximation error, and
$$G^{\pi_\theta}(\mu) = \bE\Big[\nabla_\theta\log\pi_\theta(a|s)\nabla_\theta^\top\log\pi_\theta(a|s)\Big],$$ be the Fisher information matrix under policy $\pi_\theta$, where the expectations are over $s\sim d_\mu^{\pi_\theta},a\sim \pi_\theta(\cdot|s)$. Then, we have:
\begin{equation}
    G^{\pi_\theta}(\mu)w_\lambda^{\pi_\theta} = (1-\gamma)\nabla_\theta \cV_\lambda^{\pi_\theta}(\mu),
    \end{equation}
    where
    \begin{equation}w_\lambda^{\pi_\theta} \in \arg\min_{w\in\bR^d}L(w, \theta),
    \label{eqn:cfa}
\end{equation}
for any $\theta \in\bR^d$.
\label{lemma:cfa}
\end{lemma}
The proof of Lemma \ref{lemma:cfa} can be found in Appendix \ref{app:preliminaries}.

\subsection{Entropy-Regularized NPG}
For a constant step-size $\eta > 0$, the natural policy gradient algorithm updates the parameter according to the following update \cite{kakade2001natural}:
\begin{equation}
    \theta \leftarrow \theta + \eta [G^\pit(\mu)]^{\dagger}\nabla_\theta \mathcal{V}_\lambda^\pit(\mu),
\end{equation}
where $[G^\pit(\mu)]^\dagger$ denotes the Moore-Penrose pseudoinverse of $G^\pit(\mu)$.
Equivalently, based on Lemma~\ref{lemma:cfa}, the update rule under NPG can be expressed as follows:
\begin{equation}
    \theta \leftarrow \theta + \frac{\eta}{1-\gamma} w_\lambda^{\pi_{\theta}},
\end{equation}
where $w_\lambda^{\pi_{\theta}}$ is obtained from \eqref{eqn:cfa}.
The pseudocode for NPG with a constant step-size $\eta > 0$ is given in Algorithm \ref{alg:npg}. For any $t\geq 0$, we denote $\pi_t = \pi_{\theta_t}$ throughout the paper.
\begin{algorithm}[t]
\caption{Entropy-regularized NPG}
 \label{alg:npg}
 \begin{algorithmic}
\STATE{Inputs: Step-size $\eta > 0$}
\STATE{Initialization: {$\theta_0 = \textbf{0}$ or equivalently $\pi_0(a|s)=\frac{1}{|\cA|}$ for all $(s, a)\in\cS\times\cA$}}
 \FOR{$t < T$}
    \STATE{Compute $w_t = w_\lambda^{\pi_t}$ by using \eqref{eqn:cfa}}
   \STATE{$\theta_{t+1} = \theta_t + \eta w_t$}
\ENDFOR
\end{algorithmic}
\end{algorithm}

\subsection{Entropy-Regularized NPG with Averaging}
In the following,  we introduce a slight modification of  entropy-regularized NPG with averaging and gradient clipping, summarized in Algorithm \ref{alg:q-npg}.

Starting with $\theta_0 = 0$, for a given sequence of iterates $\{\theta_k:k \leq t\}$ and step-sizes $\{\eta_t:t \geq 0\}$, the entropy-regularized NPG with averaging update is as follows:
\begin{align}
\begin{aligned}
    \theta_{t+1} &= \theta_t + \eta_t g_t,\\
    &= (1-\eta_t \lambda)\theta_t + \eta_t w_t,
\end{aligned}
    \label{eqn:q-npg-update}
\end{align}
where $g_t = w_t-\lambda\theta_t,$ 
and 
\begin{equation}
    w_t = \underset{w\in\bR^d:\|w\|_2 \leq R}{\arg\min}\bar{L}(w,\theta_t).
    \label{eqn:gradient-q-npg}
\end{equation}
for a given projection radius $R>0$. This variant of NPG is a stochastic approximation algorithm (see Remark \ref{remark:sa}). Projection step with radius $R>0$ in conjunction with the averaging in \eqref{eqn:q-npg-update} provides regularization, i.e., a direct control over $\|\theta_t\|_2$ in terms of $R$ and $\lambda$ (see Lemma \ref{lemma:q-npg}). 

Unlike \eqref{eqn:loss}, here we set 
\begin{equation*}\bar{L}(w, \theta) = \bE_{(s,a)\sim d_\mu^\pit\circ\pit}\Big[\big(w^\top \nabla\log\pit(a|s)-\Xi_\lambda^\pit(s,a)\big)^2\Big],
\end{equation*}
where $$\Xi_\lambda^\pi(s,a) = \mathcal{Q}_\lambda^\pi(s,a) - \bE_{a^\prime\sim\pi(\cdot|s)}\mathcal{Q}_\lambda^\pi(s,a^\prime).$$

\begin{algorithm}[t]
\caption{Entropy-regularized NPG with averaging}
 \label{alg:q-npg}
\begin{algorithmic}
\STATE{Inputs: Radius $R>0$, step-sizes $\{\eta_t = \frac{1}{(t+1)\lambda}: t\geq 0\}$}
 \STATE{Initialization: {$\theta_0 = \textbf{0}$ or equivalently $\pi_0(a|s)=\frac{1}{|\cA|}$ for all $(s, a)\in\cS\times\cA$}}
 \FOR{$t < T$}
    \STATE{Compute $w_t$ by using \eqref{eqn:gradient-q-npg}}
  
    \STATE{Set $g_t = w_t-\lambda\theta_t$}
   \STATE{$\theta_{t+1} = \theta_t + \eta_t g_t$}
\ENDFOR
\end{algorithmic}
\end{algorithm}

In the following, we establish a connection between NPG with averaging and stochastic approximation to provide an intuition about the algorithm.

\begin{remark}[Stochastic approximation interpretation]
Note that the updates of NPG with averaging, described in \eqref{eqn:q-npg-update}, can be rewritten as follows:
\begin{equation}
    \theta_{t+1}^\top \phi_{s,a} = (1-\bar{\eta}_t)\theta_t^\top \phi_{s,a} + \bar{\eta}_t\frac{w_t^\top \phi_{s,a}}{\lambda},~\forall (s,a)\in\cS\times\cA,
    \label{eqn:stochastic-approximation}
\end{equation}
where $\bar{\eta}_t = \eta_t\lambda$. From \eqref{eqn:gradient-q-npg}, it can be seen that $w_t^\top\phi_{s,a}\approx \cQ_\lambda^{\pi_t}(s,a)$ with respect to the weighted $\ell_2$-norm. In that respect, entropy-regularized NPG with averaging is a stochastic approximation variant with the step-size sequence $\bar{\eta}_t = \frac{1}{t+1},~t\geq 0$, which satisfies $\sum_t \eta_t = \infty$ and $\sum_t\eta_t^2 < \infty$ \cite{bertsekas1996neuro}. It is straightforward to show that if \eqref{eqn:stochastic-approximation} converges, it converges to the stationary point $\bar{\theta}\in\bR^d$ that satisfies $\phi_{s,a}^\top\bar{\theta} \approx \frac{\cQ_\lambda^{\pi_{\bar{\theta}}}(s,a)}{\lambda}$, yielding the optimal policy $\pi^*$ under the realizability assumption $\pi^* = \pi_{\theta^*}$ \cite{nachum2017bridging}. In this paper, we provide a finite-time analysis of this algorithm.

Also, note that the entropy-regularized NPG with averaging in \eqref{eqn:q-npg-update} can be written as:
\begin{equation*}
    \pi_{t+1}(a|s) = \frac{1}{Z_t(s)}\big(\pi_t(a|s)\big)^{1-\eta_t\lambda}\exp\Big(\bar{\eta}_t\frac{w_t^\top\phi_{s,a}}{\lambda}\Big),~(s,a)\in\cS\times\cA,
\end{equation*}
for some normalization term $Z_t(s)$. Since $w_t^\top\phi_{s,a} \approx \cQ_\lambda^{\pi_t}(s,a)$, the above iterations are analogous to the \textit{tabular} entropy-regularized NPG iterations analyzed in \cite{cen2020fast}. 
\label{remark:sa}
\end{remark}

\begin{remark}[Baseline]\normalfont
$\mathcal{Q}_\lambda^\pi(s,a)$ is biased in the sense that $\bE_{a\sim\pi}\mathcal{Q}_\lambda^\pi(s,a) \neq 0$.
In \eqref{eqn:gradient-q-npg}, we use $b(s) = \bE_{a\sim\pi(\cdot|s)}[\mathcal{Q}_\lambda^\pi(s,a)]$ as a baseline, which is a common variance reduction technique in policy gradient algorithms \cite{sutton2018reinforcement}.

\end{remark}

In the following section, we present the main convergence results in this paper.

\section{Main Results}\label{sec:main-results}
In this section, we establish the convergence rates for the entropy-regularized NPG methods introduced in the previous section.

First, we prove that the entropy-regularized NPG with averaging achieves $\tilde{O}(1/T)$ convergence rate under minimal assumptions in the deterministic setting. Then, we show that with additional but mild regularity conditions on the basis vectors and concentrability coefficient, the entropy-regularized NPG can achieve much faster convergence. 

\subsection{Convergence of Entropy-Regularized NPG with Averaging}
For the convergence of entropy-regularized NPG with averaging, we make the following assumption.
\begin{assumption}[Concentrability coefficient for state-visitation]
We assume that the initial state distribution $\mu$ satisfies $$\left\|\frac{d^{\pi^*}}{\mu}\right\|_\infty= \sup_{s\in\cS}\frac{d_\mu^{\pi^*}(s)}{\mu(s)} < \infty.$$

\label{assumption:conc-coeff}
\end{assumption}

\begin{lemma}[Persistence of excitation]
 For any $\lambda > 0$ and $R > 0$, the following bounds are satisfied under entropy-regularized NPG with averaging:
 \begin{equation*}
     \sup_{t\geq 0}\|\theta_t\|_2 \leq R/\lambda,
 \end{equation*}
 and
 \begin{equation*}
     \inf_{t \geq 0}\min_{(s,a)\in\cS\times\cA}\pi_t(a|s) = p_{min} \geq \frac{\exp(-2R/\lambda)}{|\cA|} > 0.
 \end{equation*}
 \label{lemma:q-npg}
\end{lemma}

\begin{proof}
Note that a recursive calculation of $$\theta_{t+1} = (1-\eta_t\lambda)\theta_t + \eta_tw_t,$$ with the diminishing step-size choice $\{\eta_t: t \geq 0\}$ leads to $$\theta_t = \frac{1}{\lambda t}\sum_{k < t}w_k,$$ for $t \geq 1$. By gradient clipping in \eqref{eqn:gradient-q-npg}, which implies $\|w_t\|_2 \leq R$, and triangle inequality, we have $\|\theta_t\|_2 \leq R/\lambda$ for all $t \geq 0$. In order to find a lower bound for $p_{min}$, observe that $$\|\theta_t^\top \phi_{s,a}\|_2 \leq R/\lambda,$$ for any $t\geq 0$ by Cauchy-Schwarz inequality since $\|\phi_{s,a}\|_2 \leq 1$ and $\|\theta_t\|_2 \leq R/\lambda$. Under softmax parameterization, this implies $\pi_t(a|s) \geq e^{-2R/\lambda}/|\cA|$ for any $t\geq 0$ and $(s,a)\in\cS\times\cA$.
\end{proof}
Lemma \ref{lemma:q-npg} implies that the policy parameter $\theta_t$ is uniformly bounded throughout the trajectory, which leads to a positive probability of exploration for all states for $\lambda > 0$. Such property is key for the convergence of policy gradient methods \cite{bhandari2019global, agarwal2020optimality}. Note that Lemma \ref{lemma:q-npg} is a direct consequence of averaging and gradient clipping (see Equation \eqref{eqn:q-npg-update}), and it still holds with probability 1 under approximate NPG, where sample-based estimation is used for finding $w_t$ at each iteration. { As the regularization coefficient $\lambda \rightarrow 0$, the minimum exploration probability $p_{min}$ also goes to 0.}

Before presenting the main result, we first give several definitions. 
\begin{definition}
For any given $R > 0$, define:
\begin{equation}
    \epsilon(R) = \sup_{t \geq 0}\min_{w:\|w\|_2 \leq R}\bar{L}(w,\theta_t),
\end{equation}
where $\bar{L}$ is the function approximation error in \eqref{eqn:gradient-q-npg}.
\label{eqn:approx-error}
\end{definition}
Note that $\epsilon(R)$ is always bounded since $\cQ_\lambda^\pit(s, a)$ is uniformly bounded  by $r_{max} + \gamma\frac{r_{max}+\lambda\log|\cA|}{1-\gamma}$, which follows from \eqref{eqn:val-bounds}.

Our analysis relies on the following Lyapunov function, which is used in mirror descent analysis in supervised learning and reinforcement learning problems \cite{martens2020new, wang2019neural, agarwal2020optimality}.

\begin{definition}[Lyapunov function]
For any $\pi\in\Pi$, we define the potential function $\Phi:\Pi\rightarrow \bR^+$ as follows:
\begin{align*}
    \Phi(\pi) &= \sum_{s\in\cS}d_\mu^{\pi^*}(s)\sum_{a\in\cA}\pi^*(a|s)\log\frac{\pi^*(a|s)}{\pi(a|s)}, \\ &= \sum_{s\in\cS}d_\mu^{\pi^*}(s)\mathcal{D}_{\mathsf{KL}}(\pi^*(\cdot|s)\|\pi(\cdot|s)).
\end{align*}
where $\mathcal{D}_{\mathsf{KL}}$ is the Kullback-Leibler divergence.
\label{def:pf}
\end{definition}
Note that $\Phi(\pi)$ is a divergence measure which measures the proximity of a policy $\pi$ to the optimal policy $\pi^*$. We have $\Phi(\pi)\geq 0$ for all $\pi$, and $\Phi(\pi) = 0$ iff $\pi(\cdot|s) = \pi^*(\cdot|s)$ for all $s\in supp(d_\mu^{\pi^*})$.

\begin{lemma}[Lyapunov drift]
 For any $t \geq 0$, consider a general update $\theta_{t+1} = \theta_t + \eta_t g_t$ under softmax parameterization with linear function approximation. Then, we have the following Lyapunov drift inequality:
\begin{multline}
    \Phi(\pi_{t+1})-\Phi(\pi_t) \leq - \eta_t\lambda\Phi(\pi_t)- \eta_t(1-\gamma)\Delta_t \\-\eta_t\bE_{(s,a)\sim d_\mu^{\pi^*}\circ\pi^*}\left[ \Big(\nabla_\theta^\top\log\pi_t(a|s)g_t-q_\lambda^{\pi_t}(s,a)\Big)\right]
    -\eta_t\bE_{s\sim d_\mu^{\pi^*}}\cV_\lambda^{\pi_t}(s) +\frac{1}{2}\eta_t^2\|g_t\|_2^2,
    \label{eqn:drift}
\end{multline}
where $\Delta_t = \cV_\lambda^{\pi^*}(\mu)-\cV_\lambda^{\pi_t}(\mu)$.
\label{lemma:lyapunov}
\end{lemma}
The negative drift term $-\lambda\eta_t\Phi(\pi_t)$ in \eqref{eqn:drift}, which stems from entropy regularization, leads to a recursion for $\{\Phi(\pi_t):t\geq 0\}$, which is key for fast convergence rates. The proof of Lemma \ref{lemma:lyapunov} can be found in Appendix \ref{app:conv-analysis}.

In the following main theorem, we show that entropy-regularized NPG with averaging achieves an improved $\tilde{O}(1/T)$ convergence rate up to the function approximation error $\epsilon(R)$. 

\begin{theorem}[Convergence of entropy-regularized NPG with averaging]
    Under Assumption \ref{assumption:conc-coeff}, for any $T > 0$, $\lambda > 0$ and $R > 0$, the entropy-regularized NPG with averaging under the step-size sequence $\eta_t = \frac{1}{\lambda(t+1)}$ achieves the following bounds:
    \begin{equation*}
        \Phi(\pi_T) \leq \frac{\sqrt{\frac{1}{T}\sum\limits_{t=1}^T\bE_{s\sim d_\mu^\star}\chi^2(\pi^*(\cdot|s)\|\pi_t(\cdot\|s))\frac{1}{1-\gamma}\|d_\mu^{\pi^*}/\mu\|_\infty\epsilon(R)}}{\lambda} + \frac{2R^2(1+\log T)}{\lambda^2 T},
        \end{equation*}
        and \begin{align*}
        \min_{0\leq t < T}\Delta_t &\leq \frac{\sqrt{\frac{1}{T}\sum\limits_{t=1}^T\bE_{s\sim d_\mu^\star}\chi^2(\pi^*(\cdot|s)\|\pi_t(\cdot\|s))\frac{1}{1-\gamma}\|d_\mu^{\pi^*}/\mu\|_\infty\epsilon(R)}}{1-\gamma}+\frac{2R^2(1+\log T)}{(1-\gamma)\lambda T},\\
        &\leq \sqrt{\frac{2\cdot\epsilon(R)\cdot\|d_\mu^{\pi^*}/\mu\|_\infty}{(1-\gamma)^3\cdot p_{min}}}+\frac{2R^2(1+\log T)}{(1-\gamma)\lambda T},
    \end{align*}
    where $\Delta_t = \cV_\lambda^{\pi^*}(\mu)-\cV_\lambda^{\pi_t}(\mu)$.
    \label{thm:q-npg}
\end{theorem}
The proof of Theorem \ref{thm:q-npg} can be found in Appendix \ref{app:conv-analysis}. {In the proof, we used the Lyapunov function in \cite{agarwal2020optimality} to show that NPG update with function approximation under entropy regularization leads to an approximate pseudo-contraction in terms of the Lyapunov function $\Phi$ with modulus $(1-\lambda_t\eta)$ with controllable extra terms, which enabled fast convergence, and also lead to the persistence of excitation condition in Lemma \ref{lemma:q-npg}, which eliminated the trajectory-dependent strong distribution mismatch assumptions in the literature (see Remark \ref{remark:dist-mismatch} also).}

We have the following observations from Theorem \ref{thm:q-npg}.

\begin{remark}[Impact of entropy regularization]\normalfont
Note that for any $R > 0$, increasing the regularization parameter (or temperature) $\lambda$ encourages the policy to be more exploratory, and $\lambda\rightarrow\infty$ leads to the maximum entropy policy as expected. On the other hand, the algorithm converges to the optimal entropy-regularized value function $\cV_\lambda^\pit(\mu)$ as $\lambda\to0$, and increasing $\lambda$ may lead to a larger error in terms of unregularized value function.
\end{remark}

\begin{remark}[Function approximation error and the tabular case]\normalfont
The representation power of the function approximation in approximating $\cQ_\lambda^\pi$ determines the optimality gap in Theorem \ref{thm:q-npg}. In Theorem \ref{thm:q-npg}, we observe that the concentrability coefficient $M$ also has an impact on the function approximation error. {A large projection radius $R$ leads to a smaller function approximation error $\epsilon(R)$, but also leads to larger $p_{min}^{-1}$ and $\frac{2R^2(1+\log T)}{(1-\gamma)\lambda T}$ terms in the upper bound. As such, the right choice of $R$, which requires the knowledge of the $\ell_2$-norm of the best parameter that leads to a good approximation of $\mathcal{Q}_\lambda^\pi$, is important, and a sharper prior knowledge for $R$ leads to better results by Theorem \ref{thm:q-npg}.} \rev{Also, note that in the tabular case, where $d=|\cS\times\cA|$ and the feature matrix with rows $\phi_{s,a}^\top$ is full-rank, the function approximation error is $0$ since there exists $w\in\bR^d$ such that $w^\top\phi_{s,a}=\cQ_\mu^{\pi_t}(s,a)$ for every $(s,a)\in\cS\times\cA$. In this case, we have $\epsilon(R)=0$ for a sufficiently large but finite $R>0$ in Theorem \ref{thm:q-npg}.}
 \end{remark}

\rev{
In the following, we show that the optimality gap scales with $\log(1/p_{min})$, which shows a sharper bound of order $O(R/\lambda)$ rather than $O(\exp(2R/\lambda))$.
\begin{proposition}
    Under Assumption \ref{assumption:dist-mismatch}, for any $T > 0, \lambda > 0$ and $R>0$, the entropy-regularized NPG with averaging under the step-size sequence $\eta_t=\frac{1}{\lambda(t+1)}$ achieves the following bound
    \begin{equation*}
        (1-\gamma)\min_{0\leq t < T}\Delta_t\leq \frac{2}{T}\sum_{t<T}\log\left(\frac{\left\|d_\mu^{\pi^*}/\mu\right\|_\infty M_t(w_t)}{p_{min}(1-\gamma)}\right)+\frac{2R^2(1+\log T)}{\lambda T},
    \end{equation*}
    where $M_t(w_t)$ is the moment-generating function of $\mathsf{err}_t(s,a)=|w_t^\top \phi_{s,a}-\cQ_\lambda^{\pi_t}(s,a)|$ with exponent 1 under the distribution $d_\mu^{\pi_t}\circ\pi_t$ at time $t$.
    \label{prop:donsker-varadhan}
\end{proposition}
The proof of Prop. \ref{prop:donsker-varadhan}, which can be found in Appendix \ref{app:conv-analysis}, follows from using a change-of-measure argument based on the Donsker-Varadhan variational principle in the Lyapunov drift (Lemma \ref{lemma:lyapunov}) to characterize distributional shift \cite{donsker1983asymptotic, hellstrom2023generalization}. Obviously, for all $t\geq 0$, $\log M_t(w_t)$ exists and is bounded by the function approximation error, i.e., $\log M_t(w_t) \leq 2\max_{s,a}|w_t^\top\phi_{s,a}-\cQ_\lambda^{\pi_t}(s,a)|\leq 2R+\frac{2(r_{max}+\lambda\log|\mathbb{A}|)}{1-\gamma}$.

Proposition \ref{prop:donsker-varadhan} shows an upper bound that scales at a rate $\log(1/p_{min})$ rather than $1/p_{min}^{1/2}$, which is important to characterize the upper bound in the regime $\lambda \rightarrow 0$. Proposition \ref{prop:donsker-varadhan} implies a linear bound in $R/\lambda$. 
}

\rev{Finally, for comparison with different distribution mismatch assumptions in the literature, we consider a variant of entropy-regularized NPG with a slightly different policy update, extending the Q-NPG (unregularized) algorithm in \cite{agarwal2020optimality} by incorporating entropy regularization. For any given distribution $\nu$ over $\cS\times\cA$ and policy $\pi$, let $\breve{d}^\star(s,a)=\frac{1}{|\cA|}d_\mu^{\pi^*}(s)$ for all $(s,a)\in\cS\times\cA$, and
\begin{align*}
    \breve{d}_\nu^\pi(s,a)&=(1-\gamma)\sum_{t=0}^\infty\bP^\pi(s_t=s,a_t=a|(s_0,a_0)\sim\nu),\\
    \breve{L}(w,\theta)&=\bE_{(s,a)\sim\breve{d}_\nu^{\pi_t}}(w^\top\phi_{s,a}-\cQ_\lambda^{\pi_t}(s,a))^2.
\end{align*}
Then, we perform the policy update \eqref{eqn:q-npg-update} with
\begin{equation}
    w_t\in\arg\min_{w\in\bR^d:\|w\|_2\leq R}\breve{L}(w,\theta_t),~t=0,1,\ldots.
    \label{eqn:er-qnpg}
\end{equation}
The following result directly follows from Lemma \ref{lemma:lyapunov} and the proof of Theorem \ref{thm:q-npg}.
\begin{proposition}[Entropy-regularized Q-NPG]
    Under the assumption that 
\begin{equation}
  \max_{s,a}\frac{d_\mu^{\pi^*}(s)/|\cA|}{\nu(s,a)}=\left\|\frac{\breve{d}^\star}{\nu}\right\|_\infty<\infty,
  \label{eqn:agarwal-assumption}
\end{equation}
the entropy-regularized Q-NPG algorithm with update \eqref{eqn:er-qnpg} for any $\lambda > 0$, $R>0$, $T > 0$ with the step-size $\eta_t=\frac{1}{\lambda(t+1)}$ achieves 
\begin{equation}
    \min_{0\leq t < T}\Delta_t\leq \frac{2}{(1-\gamma)^{3/2}}\sqrt{\breve{\epsilon}(R)\cdot|\cA|\cdot\|\breve{d}^\star/\nu\|_\infty}+\frac{2R^2(1+\log T)}{(1-\gamma)\lambda T},
    \label{eqn:agarwal-npg}
\end{equation}
where $\max_{0\leq t<T}\min_{w\in\bR^d:\|w\|_2\leq R}\breve{L}(w,\theta_t)= \breve{\epsilon}(R)$.
\label{prop:er-qnpg}
\end{proposition}
The proof of Prop. \ref{prop:er-qnpg} can be found in Appendix \ref{app:conv-analysis}. We note that one can also use the unbiased version with updates $w_t\in\arg\min_w\bE_{(s,a)\sim\breve{d}_\nu^{\pi_t}}(w^\top\nabla\log\pi_t(a|s)-\Xi_\lambda^{\pi_t}(s,a))^2$ that yields a similar bound.

\begin{remark}[Concentrability coefficient and exploration]\normalfont
Our Assumption \ref{assumption:conc-coeff} is significantly milder compared to the distribution mismatch assumptions in the literature (see \cite{scherrer2014approximate, shani2020adaptive} for a discussion), and it is stated that such an exploratory initial state distribution $\mu$ is indeed necessary for the convergence of policy gradient methods \cite{bhandari2019global}. In the existing works, the convergence results are established under a strong assumption that
\begin{equation}
\bE_{s\sim d_\mu^{\pi_t}, a\sim \pi_t(\cdot|s)}\Big[\Big(\frac{d_\mu^{\pi^*}(s)\pi^*(a|s)}{d_\mu^{\pi_t}(s)\pi_t(a|s)}\Big)^2\Big]<\infty,~t=0,1,\ldots,
\label{eqn:assumption-strong-cc}
\end{equation}
 which assumes that $\pi_t$ performs sufficient exploration at each policy optimization step $t=0,1,\ldots$ to ensure $\pi^*(\cdot|s)\ll\pi_t(\cdot|s),~s\in\cS$  \cite{wang2019neural, liu2019neural, chen2019information, fan2020theoretical}. On the other hand, our Assumption \ref{assumption:conc-coeff} is completely independent of the policy trajectory $\{\pi_t\}$ under the NPG, as it is basically an assumption \emph{only} on the initial state distribution $\mu$. The key result to prove convergence under this weak Assumption \ref{assumption:conc-coeff} is the persistence of excitation in Lemma \ref{lemma:q-npg}, which asserts that entropy-regularized NPG performs sufficient exploration to ensure convergence, rather than assuming that NPG iterates perform exploration in the form of \eqref{eqn:assumption-strong-cc}.

 The convergence result in \eqref{eqn:agarwal-npg}, which is an extension of Q-NPG in \cite{agarwal2020optimality} to entropy regularization, is under a different distribution mismatch assumption \eqref{eqn:agarwal-assumption}. In this case, the convergence of this algorithm heavily relies on the exploratory nature of the initial state-action distribution $\nu$. Note that this variant is not exactly the NPG, since the original NPG update $[G^{\pi_t}(\mu)]^\dagger\nabla\cV_\lambda^{\pi_t}(\mu)$ is the solution of \eqref{eqn:cfa} under the distribution $d_\mu^{\pi_t}\circ\pi_t$ (by Lemma \ref{lemma:cfa}) while the variant in \eqref{eqn:er-qnpg} computes the policy update under a different state-action distribution $\breve{d}_\nu^{\pi_t}$. Accordingly, the distribution mismatch and function approximation notions differ considerably with respect to the original NPG that we analyzed in Theorem \ref{thm:q-npg}.
 
 \label{remark:dist-mismatch}
\end{remark}

}
 
 \subsection{Linear Convergence of Entropy-Regularized NPG}
Under additional regularity assumptions on the distribution mismatch and basis vectors compared to Theorem \ref{thm:q-npg}, we will prove that entropy-regularized NPG achieves linear convergence rate $O(e^{-\Omega(T)})$ up to the compatible function approximation error.

The following assumption is standard in reinforcement learning literature \cite{wang2019neural,liu2019neural}.
\begin{assumption}[Concentrability coefficient]
Let the concentrability coefficient be defined as
$$C_t = \chi^2(d_\mu^{\pi^*}\circ\pi^*\|d_\mu^{\pi_t}\circ\pi_t).$$
We assume that there exists a constant $C^\star < \infty$ such that $C_t \leq C^\star$ for all $t$.
\label{assumption:dist-mismatch}
\end{assumption}
Note that Assumption \ref{assumption:dist-mismatch} is stronger than Assumption \ref{assumption:conc-coeff} as it requires exploratory behavior of the policies throughout the iterations.

We will prove linear convergence under a mild regularity condition on the parametric model \eqref{eqn:policy-class}, which we present in the following.
\begin{assumption}[Regularity of the parametric model]
We assume that $G^{\pi_0}(\mu)$ is non-singular where $\pi_0(a|s) = 1/|\cA|$ for all $(s, a)\in\cS\times\cA$.
\label{assumption:basis}
\end{assumption}
Assumption \ref{assumption:basis} is a regularity condition on the basis functions $\{\phi_{s,a}:(s,a)\in\cS\times\cA\}$. Similar regularity conditions, such as boundedness of the relative condition number, are assumed in the RL literature \cite{agarwal2020optimality}.

\begin{remark}[Regularity of random features]\normalfont
An important class of basis vectors is random features, which have fundamental importance in kernel-based estimation and the analysis of neural networks \cite{rahimi2007random, jacot2018neural}. In the following, we consider an example of random features, and show that Assumption \ref{assumption:basis} holds with high probability in the function approximation regime.
\begin{proposition}[Regularity of Gaussian Random Features]
Consider an ensemble of random features with $\phi_{s,a} \sim \mathcal{N}(0,I_d)$ for all $(s,a)\in\cS\times\cA$. For $\mu = \mathsf{Unif}(\cS)$, $|\cA| =  2$ and $\delta\in(0,1)$, we have:
\begin{equation*}
    \sigma_1(G^{\pi_0}(\mu)) \geq \frac{1-\gamma}{8}\Big(\frac{1}{2}-\sqrt{\frac{\log(1/\delta)}{2|\cS|}}-\sqrt{\frac{16d\log(|\cS|)}{|\cS|}}\Big),
\end{equation*}
with probability at least $1-\delta$.
\label{prop:gaussian-regularity}
\end{proposition}
Proposition \ref{prop:gaussian-regularity} implies that in the function approximation setting where $d \ll |\cS\times \cA|$, the ensemble of random feature vectors satisfies the regularity condition in Assumption \ref{assumption:basis} with high probability. The analysis is based on Rademacher complexity bounds, and can be used to extend Proposition \ref{prop:gaussian-regularity} to general $\mu$ and $\cA$. We prove Proposition \ref{prop:gaussian-regularity}, and also numerically investigate the regularity of the neural tangent kernel (NTK) features in Appendix \ref{app:conv-analysis}.
\end{remark}

\begin{lemma}[Non-singularity lemma]
    Under Assumptions \ref{assumption:dist-mismatch}-\ref{assumption:basis}, there exists a constant $\sigma > 0$ such that the following holds under Algorithm \ref{alg:npg}:
    \begin{equation}
        \inf_{t\geq 0}\min_{i\in[d]}\sigma_i\Big(G^{\pi_t}(\mu)\Big) \geq \sigma,\end{equation}
    where the constant step-size is:
    \begin{equation}
        \eta \leq \min\Big\{ \frac{(1-\gamma)\sigma^2r_{min}}{(r_{max}+\lambda\log|\cA|)^2},\frac{1}{2\lambda}\Big\}.
        \label{eqn:step-size}
    \end{equation}
    \label{lemma:prob-bounds}
\end{lemma}
Lemma \ref{lemma:prob-bounds} implies $G^{\pi_t}(\mu)$ is strictly positive definite for all $t\geq 0$ because of entropy regularization.

By using the results of Lemma \ref{lemma:prob-bounds}, we have the following result on the linear convergence of entropy-regularized NPG.
\begin{theorem}[Convergence of entropy-regularized NPG]
Under Assumptions \ref{assumption:dist-mismatch}-\ref{assumption:basis}, the entropy-regularized NPG with the constant step-size in \eqref{eqn:step-size} satisfies the following:
\begin{equation*}
    \Phi(\pi_T) \leq (1-\eta\lambda)^T\log|\cA| + \frac{\sqrt{C^\star\cdot \ea}}{\lambda},
\end{equation*}
    and
\begin{align}
    \min_{0\leq t< T}\Delta_t &\leq \frac{\lambda\rho^T\log|\cA|}{(1-\gamma)\cdot\big(1-\rho^T\big)}+ \frac{\sqrt{C^\star \ea}}{1-\gamma}, \label{eqn:lc-best}\\
     \Delta_T &\leq\frac{ \big(1-\eta\lambda\big)^T\log|\cA|}{\eta(1-\gamma)} + \frac{\sqrt{C^\star\cdot\ea }}{\lambda\eta(1-\gamma)}\label{eqn:lc-last},
\end{align}
for any $T > 0$, where $\Delta_t = \cV_\lambda^{\pi^*}(\mu)-\cV_\lambda^{\pi_t}(\mu)$, $\rho = 1-\eta\lambda$, and $$\ea = \sup_{t \geq 0}\min_{w\in\bR^d}L(w,\theta_t),$$ for the loss function $L$ defined in \eqref{eqn:loss}.
\label{thm:main}
\end{theorem}
Note that the error term $\ea$ is the compatible function approximation error in \eqref{eqn:cfa}, which has different characteristics than the function approximation error $\epsilon(R)$ in \eqref{eqn:approx-error}. The proof of Theorem \ref{thm:main} is given in Appendix
\ref{app:conv-analysis}.
\begin{remark}[Last iterate convergence]\normalfont
In Theorem \ref{thm:main}, we provide convergence bounds for the last iterate in \eqref{eqn:lc-best} in addition to the best iterate in \eqref{eqn:lc-last}.
\end{remark}
\begin{remark}\normalfont 

    Note that $\lambda\downarrow 0$ implies $O(1/T)$ convergence rate up to the function approximation error $\frac{\sqrt{C^\star \ea}}{1-\gamma}$ asymptotically by \eqref{eqn:lc-best} in Theorem \ref{thm:main}, which implies $O\left(1/T\right)$ convergence rate in the unregularized MDP. This convergence rate matches the convergence rate for the unregularized MDP in \cite{xu2020improving}, { and is faster than the convergence rate $O(1/\sqrt{T})$ in \cite{agarwal2020optimality}} under similar concentrability conditions.
\end{remark}

\begin{corollary}
The following bound is satisfied under Assumptions \ref{assumption:dist-mismatch}-\ref{assumption:basis}:
\begin{equation}
    \min_{0\leq t< T}\Delta_t \leq \frac{ \big(1-\eta\lambda\big)^T\log|\cA|}{\eta(1-\gamma)}+ \frac{\sqrt{C^\star \ea}}{1-\gamma},
    \end{equation}
    for any $T \geq 1$.
\end{corollary}
\begin{proof}
The proof follows by substituting the bound $\frac{\lambda}{1-\rho^T}\leq \frac{\lambda}{1-\rho}\leq \frac{1}{\eta}$ into \eqref{eqn:lc-best}.
\end{proof}

\section{Sample-Based NPG with Entropy Regularization}\label{sec:nac}
The convergence results in Section \ref{sec:main-results} are based on the exact knowledge of $w_t$ at each iteration to understand the dynamics of the entropy-regularized NPG methods in terms of iteration complexity and function approximation error. In practice, $w_t$ should be estimated by solving \eqref{eqn:cfa} and \eqref{eqn:gradient-q-npg} using samples, which introduces statistical errors. In this section, we characterize the impact of statistical errors on the convergence of NPG methods.

Let $\mathcal{F}_t$ be the sigma-field generated by all samples used until (excluding) iteration $t$.

\textbf{Critic:} Consider the following Bellman operator:
\begin{equation}
    \mathcal{T}^\pi q(s, a) = r(s, a)-\lambda\log\pi(a|s) + \gamma \bE_{s^\prime a^\prime}q(s^\prime,a^\prime),
\end{equation}
for any $q:\cS\times \cA\rightarrow \bR$, where the expectation is over $s^\prime \sim P(\cdot|s,a)$ and $a^\prime\sim\pi(\cdot|s)$. Note that the shifted Q-function, $q_\lambda^\pi$, is the fixed point of the Bellman equation: $$q_\lambda^\pi(s,a) = \mathcal{T}^\pi q_\lambda^\pi(s,a),$$ whereas $\cQ_\lambda^\pi$ does not directly satisfy it. Therefore, we estimate $q_\lambda^\pi$,  which is the fixed point of the Bellman equation, and then use the relation $$\cQ_\lambda^\pi(s,a) = q_\lambda^\pi(s,a) + \lambda\log\pi(a|s),$$ to obtain a sample-based estimate for $\cQ_\lambda^\pi$. In order to find the fixed point of the Bellman equation, temporal difference learning with function approximation provides an effective method in large state-action spaces \cite{tsitsiklis1997analysis, sutton1988learning, bhandari2018finite}. We assume that the critic provides an estimate $\widehat{q}_\lambda^{\pi_t}$ for $q_\lambda^{\pi_t}$ such that:
\begin{equation}
    \bE\left[\left(q_\lambda^{\pi_t}(s,a)-\widehat{q}_\lambda^{\pi_t}(s,a)\right)^2\Big|\mathcal{F}_t\right] \leq \epsilon_{critic},
    \label{eqn:stat-error-1}
\end{equation}
which implies $\bE[(\cQ_\lambda^{\pi_t}(s,a)-\widehat{Q}_\lambda^{\pi_t}(s,a))^2|\mathcal{F}_t]\leq \epsilon_{critic}$ where $\widehat{Q}_\lambda^{\pi_t}(s,a) = \widehat{q}_\lambda^{\pi_t}(s,a)+\lambda\log\pi_t(a|s)$.

\textbf{Actor:} Given $$\widehat{\Xi}_\lambda^{\pi_t}(s,a) = \widehat{Q}_{\lambda}^{\pi_t}(s,a)-\bE_{a^\prime}\widehat{Q}_{\lambda}^{\pi_t}(s,a^\prime),$$ one needs to solve \eqref{eqn:cfa} (or \eqref{eqn:gradient-q-npg}) to find the NPG update $w_t$ by using samples $(s,a)\sim d_\mu^{\pi_t}\circ\pi_t$. This can be accomplished by stochastic gradient descent \cite{agarwal2020optimality}, or random-design least squares approach \cite{hsu2012random}. We assume that the actor update algorithm provides a gradient update $\widehat{w}_t$ which satisfies the following:
\begin{equation}
    \bE[(\widehat{w}_t^\top \log\pi_t(a|s)-\widehat{\Xi}_\lambda^{\pi_t}(s,a))^2|\mathcal{F}_t] \leq \epsilon_{actor} + \min_w \bE[(w^\top \log\pi_t(a|s)-\widehat{\Xi}_\lambda^{\pi_t}(s,a))^2|\mathcal{F}_t].
    \label{eqn:stat-error-2}
\end{equation}
at each iteration $t\leq T$.

For the black-box actor and critic algorithms that lead to the statistical errors in \eqref{eqn:stat-error-1} and \eqref{eqn:stat-error-2}, the sample-based natural policy gradient method based on the entropy-regularized NPG with averaging yields the following result.
\begin{proposition}
    Under Assumption \ref{assumption:conc-coeff}, the natural actor-critic algorithm with temperature $\lambda > 0$ yields the following bound:
    \begin{equation*}
    \bE[\min_{0\leq t < T}\Delta_t] \leq \frac{4\sqrt{\frac{2}{p_{min}}\|d_\mu^{\pi^*}/\mu\|_\infty\epsilon_{total}(R)}}{(1-\gamma)^{3/2}}+\frac{2R^2\log(T)}{(1-\gamma)\lambda T},
\end{equation*}
where $$\epsilon_{total}(R) \leq \epsilon(R) + \epsilon_{actor} + \epsilon_{critic},$$
for the statistical error terms $\epsilon_{critic}$ and $\epsilon_{actor}$ in \eqref{eqn:stat-error-1} and \eqref{eqn:stat-error-2}, respectively.
\label{prop:sample-based-npg}
\end{proposition}

In {the following subsection}, we explicitly characterize the sample complexity of an actor-critic method based on the entropy-regularized NPG with averaging, which uses temporal difference learning with linear function approximation for the critic, and stochastic gradient descent.

\subsection{Sample-Based Entropy-Regularized Natural Actor-Critic}
In the following, we consider a natural actor-critic algorithm for sample-based policy optimization based on the entropy-regularized NPG with averaging, TD learning with linear function approximation \cite{bhandari2018finite} and stochastic gradient descent (SGD) \cite{shalev2014understanding}.

\textbf{Critic:} We use temporal difference (TD) learning with linear function approximation for the critic \cite{sutton1988learning, tsitsiklis1997analysis, bhandari2018finite, wang2019neural}.

Consider the following Bellman operator:
\begin{equation}
    \mathcal{T}^\pi q(s, a) = r(s, a)-\lambda\log\pi(a|s) + \gamma \bE_{s^\prime a^\prime}q(s^\prime,a^\prime),
\end{equation}
for any $q:\cS\times \cA\rightarrow \bR$, where the expectation is over $s^\prime \sim P(\cdot|s,a)$ and $a^\prime\sim\pi(\cdot|s)$. Note that the shifted Q-function, $q_\lambda^\pi$, is the fixed point of the Bellman equation: $$q_\lambda^\pi(s,a) = \mathcal{T}^\pi q_\lambda^\pi(s,a),$$ whereas $\cQ_\lambda^\pi$ does not directly satisfy it. Therefore, for the actor-critic method, we estimate $q_\lambda^\pi$ by using TD learning with linear function approximation, and then use the relation $$\cQ_\lambda^\pi(s,a) = q_\lambda^\pi(s,a) + \lambda\log\pi(a|s),$$ to obtain a sample-based estimate for $\cQ_\lambda^\pi$.

Let $\{\psi_{s,a}\in\bR^d:(s,a)\in\cS\times\cA\}$ be the set of basis vectors for TD learning with linear function approximation. The goal in temporal difference learning is to minimize the mean-squared projected Bellman error \cite{bhandari2018finite, sutton1988learning}:
\begin{equation*}
    \beta_t^\star = \arg\min_{\beta \in\mathcal{B}(0,R)}\bE_{(s,a)\sim\nu_t}[(\mathcal{T}^{\pi_t}(\beta^\top\psi_{s,a})-(\beta^\top\psi_{s,a}))^2].
\end{equation*}
where $\mathcal{B}(0,R) = \{x\in\bR^d:\|x\|_2\leq R\}$ for a given projection radius $R > 0$.

Starting with $\beta_{0,t} = 0$, the TD learning iterates as follows \cite{bhandari2018finite, tsitsiklis1997analysis, sutton1988learning}:
\begin{align*}
    \beta_{k+\frac{1}{2},t} &= \beta_{k,t} + \alpha_k\Big(r_k^\lambda + \gamma \beta_{k,t}^\top \psi_{s_k^\prime,a_k^\prime}-\beta_{k,t}^\top\psi_{s_k,a_k}\Big)\psi_{s_k,a_k},\\
    \beta_{k+1,t} &= \mathcal{P}_{\mathcal{B}(0,R)}\big\{\beta_{k+\frac{1}{2},t}\big\},
\end{align*}
where $(s_k,a_k)\sim \nu_{\pi_t}$, $s_{k}^\prime\sim P(\cdot|s_k,a_k)$, $a_k^\prime\sim \pi_t(\cdot|s_k^\prime)$, $r_k^\lambda = r(s_k,a_k)-\lambda\log\pi_t(a_k|s_k)$, and $\mathcal{P}_\mathcal{C}$ is the projection operator onto $\mathcal{C}\subset \bR^d$. Then, the output of the TD learning is the following: $$\widehat{Q}_{\lambda,K}^{\pi_t}(s,a) = \frac{1}{K}\sum_{k \leq K}\beta_{k,t}^\top\psi_{s,a} + \lambda\log\pi_t(a|s) = \widehat{q}_{\lambda}^{\pi_t}(s,a)+ \lambda\log\pi_t(a|s),$$ for $K > 0$.

In order to characterize the sample complexity to achieve a target error $\epsilon_{critic}$, we make the following assumptions for TD learning.
\begin{assumption}
Assume the Markov chain $\{(s_k,a_k):s_{k+1}\sim P(\cdot|s_k,a_k),a_k\sim\pi_t(\cdot|s_k),k \geq 0, s_0\sim\mu\}$ is irreducible and aperiodic {with stationary distribution $\nu_{\pi_t}$} for all $t\geq 0$. Also, for all $t \geq 0$, $\nu_{\pi_t} \ll {d_\mu^{\pi_t}\otimes\pi_t}$ with a Radon-Nikodym derivative upper bounded by a constant $M^*_{st}$ where $\nu_\pi$ is the stationary state distribution under a policy $\pi$.
\label{assumption:ergodicity}
\end{assumption}
Assumption \ref{assumption:ergodicity} implies that the Markov chain under $\pi_t$ is ergodic, therefore has a stationary distribution $\nu_{\pi_t}$. For simplicity, we assume that i.i.d. samples $(s_k,a_k)$ from $\nu_{\pi_t}$ can be obtained. The second part of the assumption, i.e., $\nu_{\pi_t} \ll { d_\mu^{\pi_t}\otimes\pi_t}$, { implies that $\frac{d_\mu^{\pi_t}(s)\pi_t(a|s)}{\nu_{\pi_t}(s,a)}\leq M_{st}^*$ for any $t\geq 0$ and $(s,a)\in\mathbb{S}\times\mathbb{A}$ such that $d_\mu^{\pi_t}(s)\pi_t(a|s) > 0$.}

For simplicity, we make the following realizability assumption. Note that without this assumption, there will be an additional function approximation error since $q_\lambda^\pit$ may not be in the function class determined by $\{\psi_{s,a}:(s,a)\in\cS\times\cA\}$, which can be easily incorporated into the bound.
\begin{assumption}[Realizability]
    For any $t \geq 0$, there exists $\overline{\beta}_t\in\bR^d$ such that $\|\overline{\beta}_t\|_2 \leq \overline{R}$, and $$q_\lambda^{\pi_t}(s,a) = \overline{\beta}_t^\top\psi_{s,a},~\forall (s,a)\in\cS\times\cA,$$  for some $\overline{R} < \infty$.
    \label{assumption:realizability}
\end{assumption}

The performance of the critic is characterized by the following finite-time bounds for TD learning with linear function approximation in \cite{bhandari2018finite}.
\begin{proposition}[Theorems 2-3 in \cite{bhandari2018finite}]
    Under Assumptions \ref{assumption:ergodicity}-\ref{assumption:realizability}, 
    \begin{itemize}
        \item with constant step-size $\alpha_k = 1/\sqrt{K}$,
        \begin{equation}
            \bE[\|q_\lambda^{\pi_t}-\widehat{q}_\lambda^{\pi_t}\|_{\nu_{\pi_t}}^2|\mathcal{F}_t]\leq O\left((R-\lambda\log p_{min})^2\right)\frac{1}{(1-\gamma)\sqrt{K}},
        \end{equation}
        \item  with decaying step-size $\alpha_k = \frac{1}{\omega(k+1)(1-\gamma)}$, we have:
    \begin{equation}
        \bE[\|q_\lambda^{\pi_t}-\widehat{q}_\lambda^{\pi_t}\|_{\nu_{\pi_t}}^2|\mathcal{F}_t] \leq  O\left((R-\lambda\log p_{min})^2\right)\frac{1+\log K}{(1-\gamma)^2K\omega},
    \end{equation}
    \end{itemize}
    where $R > \overline{R}$, and $\omega$ is the minimum eigenvalue of $\sum_{s,a}\nu_{\pi_t}(s,a)\psi_{s,a}\psi_{s,a}^\top$.
    \label{prop:critic}
\end{proposition}
Hence, we have $$\epsilon_{critic} \leq M^*_{st}\cdot \bE[\|q_\lambda^{\pi_t}-\widehat{q}_\lambda^{\pi_t}\|_{\nu_{\pi_t}}^2|\mathcal{F}_t],$$ for all $t$, where $\bE[\|q_\lambda^{\pi_t}-\widehat{q}_\lambda^{\pi_t}\|_{\nu_{\pi_t}\pi_t}^2|\mathcal{F}_t]$ is characterized in Prop. \ref{prop:critic}. The factor $M^*_{st}$ comes from a change of measure argument under Assumption \ref{assumption:ergodicity}. Under regularity conditions for the features $\{\psi_{s,a}:(s,a)\in\cS\times\cA\}$, Proposition \ref{prop:critic} implies that TD learning with decaying step-size requires $K = \tilde{O}(1/\epsilon^2)$ samples at each iteration $t\leq T$ to achieve $\epsilon_{critic} = \epsilon^2$. Without any assumptions on the feature matrix, $K = O(1/\epsilon^4)$ samples at each iteration $t\leq T$ are required to achieve the same critic error $\epsilon_{critic}$ with a constant step-size.

\textbf{Actor update:} Let $$\widehat{\Xi}_\lambda^{\pi_t}(s,a) = \widehat{Q}_\lambda^{\pi_t}(s,a)-\bE_{a^\prime\sim\pi_t(\cdot|s)}\widehat{Q}_\lambda^{\pi_t}(s,a^\prime).$$ Then, starting from $\bar{w}_{0,t} = 0$, the following stochastic gradient descent (SGD) iterations are followed for $n<N$:
\begin{align*}
    \bar{w}_{n+\frac{1}{2},t} &= \bar{w}_{n,t} - \alpha^\prime\big(\nabla^\top\log\pi_t(a_n|s_n)\bar{w}_{n,t}-\widehat{\Xi}_\lambda^{\pi_t}(s_n,a_n)\big)\nabla\log\pi_t(a_n|s_n),\\
    \bar{w}_{n+1,t} &= \mathcal{P}_{\mathcal{B}(0,R)}\{\bar{w}_{n+\frac{1}{2},t}\},
\end{align*}
where $(s_n,a_n)\sim d_\mu^{\pi_t}\circ\pi_t$. Sampling from the state-action visitation distribution can be performed by using the sampler in \cite{agarwal2020optimality, konda2000actor}. Note that the estimates $\widehat{\Xi}_\lambda^{\pi_t}(s_n,a_n)$ are obtained by the critic here, unlike the unbiased sampling procedure for the unregularized Q-function in \cite{agarwal2020optimality}, which typically yields lower variance \cite{bhatnagar2009natural}.

\begin{proposition}[Theorem 14.8 in \cite{shalev2014understanding}] The above SGD iterations with the constant step-size $\alpha^\prime = R /\sqrt{ q_{max} N}$ yield the following result:
    \begin{equation}
        \epsilon_{actor} \leq \frac{R q_{max}}{\sqrt{N}},
    \end{equation}
    for any $R > \overline{R}$ where the expectation is over the random samples $\{(s_n,a_n):n\in[N]\}$, and $q_{max} = r_{max} + \frac{r_{max}+\lambda\log|\cA|}{1-\gamma}$.
    \label{prop:sgd}
    \end{proposition}

Hence, the entropy-regularized NAC performs the actor update as follows:
\begin{equation}
    \theta_{t+1} = (1-\eta_t\lambda)\theta_t + \frac{\eta_t}{N}\sum_{n=1}^N\bar{w}_{n,t}.
\end{equation}

\begin{remark}[Sample complexity of entropy-regularized NPG]
    {As a direct consequence of Propositions \ref{prop:sample-based-npg}, \ref{prop:critic} and \ref{prop:sgd}, the overall sample complexity of the sample-based natural actor-critic with TD learning and stochastic gradient descent is $O(1/\epsilon^5)$. We note that, under full-rank assumptions on the feature matrices formed by the feature vectors $[\psi_{s,a}]_{(s,a)\in\cS\times\cA}$ in the critic and $[\nabla\log\pi_t(a|s)]_{(s,a)\in\cS\times\cA}$ in the actor, the convergence rates of the actor and critic steps can be improved to $\tilde{O}(1/N)$ and $\tilde{O}(1/K)$ with diminishing step-sizes, respectively \cite{bach2021learning, bhandari2018finite}, which would imply an overall sample complexity of $\tilde{O}(1/\epsilon^3)$ for the sample-based natural actor-critic by Proposition \ref{prop:sample-based-npg}.}
\end{remark}
\section{Conclusion and Future Work}
In this work, we analyzed the convergence of natural policy gradient under softmax parameterization with linear function approximation, and established sharp finite-time convergence bounds. In particular, we proved that entropy-regularized NPG with linear function approximation achieves $\tilde{O}(1/T)$ convergence rate under only a mild distribution mismatch assumption, and achieves linear convergence rate under regularity assumptions on the basis vectors, which is significantly faster than the sublinear rates previously obtained in the function approximation setting. Based on a Lyapunov drift analysis, we proved that entropy regularization encourages exploration so that all actions are explored with a probability bounded away from zero, which explains the empirical success of entropy regularization in NPG methods with function approximation.

An immediate future work is to use the techniques that we established in this paper to improve sample complexity and overparameterization bounds for sample-based NPG with neural network approximation in the NTK regime, and study the role of entropy regularization in that setting. Another interesting future direction is the study of (vanilla) PG methods with entropy regularization in the function approximation regime.

\appendix
\section{Omitted Proofs}\label{app:preliminaries}

\subsection{Proof of Proposition \ref{thm:policy-gradient}}
We have:
\begin{equation}
    \cV_\lambda^\pit(s_0) = \sum_{a\in\cA}\pit(a|s_0)\Big(\cQ_\lambda^\pit(s_0,a)-\lambda\log\pit(a|s_0)\Big),
\end{equation}
for any $s_0\in\cS$. Taking the gradient of the above identity,
\begin{multline*}
    \nabla_\theta \cV_\lambda^\pit(s_0) = \sum_a \nabla_\theta \pit(a|s_0)\Big(\cQ_\lambda^\pit(s_0,a)-\lambda\log\pit(a|s_0)\Big)\\+\sum_a\pit(a|s_0)\Big(\nabla_\theta \cQ_\lambda^\pit(s_0,a)-\lambda\frac{\nabla_\theta\pit(a|s_0)}{\pit(a|s_0)}\Big),
\end{multline*}
since $\|\theta\|_2<\infty$ and $\pit(a|s) > 0$ for all $s,a$, $\nabla_\theta \pit(a|s) = \nabla_\theta\log\pit(a|s)\pit(a|s)$. First, note that $$\sum_a\nabla_\theta \pit(a|s) = \nabla_\theta \sum_a \pit(a|s) = 0,$$ for any $s,a$. Therefore,
\begin{equation}
    \nabla_\theta \cV_\lambda^\pit(s_0) = \sum_a \nabla_\theta \pit(a|s_0)\Big(\cQ_\lambda^\pit(s_0,a)-\lambda\log\pit(a|s_0)\Big) + \sum_a\pit(a|s_0)\nabla_\theta \cQ_\lambda^\pit(s_0,a).
    \label{eqn:grad-a}
\end{equation}
Recall that:
\begin{equation*}
    \cQ_\lambda^\pit(s, a) = r(s,a) + \sum_{s^\prime\in\cS}\cP(s^\prime|s,a)\cV_\lambda^\pit(s^\prime).
\end{equation*}
Thus, the gradient of $\cQ_\lambda^\pit(s,a)$ is as follows:
\begin{equation*}
    \nabla_\theta \cQ_\lambda^\pit(s,a) = \sum_{s^\prime\in\cS}\cP(s^\prime|s,a)\nabla_\theta \cV_\lambda^\pit(s^\prime).
\end{equation*}
Substituting this into \eqref{eqn:grad-a}, we obtain:
\begin{multline*}
    \nabla_\theta \cV_\lambda^\pit(s_0) = \sum_a \nabla_\theta \pit(a|s_0)\Big(\cQ_\lambda^\pit(s_0,a)-\lambda\log\pit(a|s_0)\Big)\\+\gamma\sum_s \bP_{\pit}(s_1=s|s_0)\nabla_\theta \cV_\lambda^\pit(s).
\end{multline*}
By induction,
\begin{align*}
    \nabla_\theta \cV_\lambda^\pit(s_0) &= \sum_{s,a}\sum_{t=0}^\infty \gamma^t \bP_\pit(s_t=s|s_0)\nabla_\theta \pit(a|s)\Big(\cQ_\lambda^\pit(s,a)-\lambda\log\pit(a|s)\Big),\\
    &= \frac{1}{1-\gamma}\sum_{s,a}d_{s_0}^\pit(s)\nabla_\theta\pit(a|s)\Big(\cQ_\lambda^\pit(s,a)-\lambda\log\pit(a|s)\Big),
\end{align*}
by the definition of $d_{s_0}^\pit(s)$. By taking expectation of the above identity over the initial state $s_0\sim\mu$ and using $\nabla_\theta \pit(a|s) = \nabla_\theta\log\pit(a|s)\pit(a|s)$, we conclude the proof.

\subsection{Proof of Lemma \ref{lemma:cfa}}
We have:
\begin{equation*}
    L(w, \theta) = \bE_{s\sim d_\mu^\pit,a\sim\pit(\cdot|s)}\Big[\Big(\nabla_\theta^\top \log\pit(a|s) w - \big(\cQ_\lambda^\pit(s, a)-\lambda\log\pit(a|s)\big)\Big)^2\Big],
\end{equation*}
for any given $w,\theta\in\bR^d$. The first-order optimality condition yields:
\begin{align*}
    \nabla_w L(w,\theta)\big|_{w=w_t} &= \bE_{s\sim d_\mu^\pit,a\sim\pit(\cdot|s)}\Big[\nabla_\theta\log\pit(a|s)\Big(\nabla_\theta^\top \log\pit(a|s) w_t - q_\lambda^\pit(s, a)\Big)\Big],\\
    &= G^\pit(\mu)w_t - (1-\gamma)\nabla_\theta \cV_\lambda^\pit(\mu) = 0,
\end{align*}
by the definition of $G^\pit(\mu)$ and Proposition \ref{thm:policy-gradient}. The result directly follows from the first-order optimality condition.

\section{Convergence Analysis of Entropy-Regularized NPG}\label{app:conv-analysis}
In this section, we will prove Theorem \ref{thm:q-npg} and Theorem \ref{thm:main}. First, we prove the Lyapunov drift lemma, which will be central in both proofs.
\subsection{Proof of Lemma \ref{lemma:lyapunov}}
The following lemmas will be useful in the proof.
\begin{lemma}[Performance difference lemma]
 For any $\theta,\theta^\prime\in\bR^d$ and $\mu\in\Delta(\cS)$, we have:
 \begin{equation}
     \cV_\lambda^{\pi_\theta}(\mu)-\cV_\lambda^{\pi_{\theta^\prime}}(\mu) = \frac{1}{1-\gamma}\bE_{s\sim d_\mu^{\pi_\theta},a\sim\pi_\theta(\cdot|s)}\Big[A_\lambda^{\pi_{\theta^\prime}}(s,a)+\lambda\log\frac{\pi_{\theta^\prime}(a|s)}{\pi_\theta(a|s)}\Big],
 \end{equation}
 where $A_\lambda^{\pi_\theta}$ is the (soft) advantage function:
\begin{equation}
    A_\lambda^\pi(s, a) = \cQ_\lambda^\pi(s, a) - \cV_\lambda^\pi(s) - \lambda\log\pi(a|s).
    \label{eqn:adv}
\end{equation}
\label{lemma:pdl}
\end{lemma}
\begin{proof}
For any $s_0\in\cS$, we have:
\begin{align*}
    \cV_\lambda^{\pit}(s_0)-\cV_\lambda^{\pitp}(s_0) &= \bE_\pit\Big[\sum_{t=0}^\infty\gamma^t\Big(r_t-\lambda\log\pit(a_t|s_t)\Big)\Big|s_0\Big]-\cV_\lambda^\pitp(s_0),\\
    &= \bE_\pit\Big[\sum_{t=0}^\infty\gamma^t\Big(r_t-\lambda\log\pit(a_t|s_t)+\cV_\lambda^\pitp(s_t)-\cV_\lambda^\pitp(s_t)\Big)\Big|s_0\Big]\\ &\hskip 8.12cm -\cV_\lambda^\pitp(s_0),\\
    &= \bE_\pit\Big[\sum_{t=0}^\infty\gamma^t\Big(r_t-\lambda\log\pit(a_t|s_t)+\gamma \cV_\lambda^\pitp(s_{t+1})-\cV_\lambda^\pitp(s_t)\Big)\Big|s_0\Big],
\end{align*}
where $r_t = r(s_t,a_t)$ and the last identity holds since $$\sum_{t=0}^\infty \gamma^t \cV_\lambda^\pitp(s_t) = \cV_\lambda^\pitp(s_0)+\gamma\sum_{t=0}^{\infty}\gamma^tV_\lambda^\pitp(s_{t+1}).$$ Then, letting $r_t = r(s_t,a_t)$ and by using law of iterated expectations,
\begin{multline*}
    \cV_\lambda^{\pit}(s_0)-\cV_\lambda^{\pitp}(s_0) \\ = \bE_\pit\Big[\sum_{t=0}^\infty\gamma^t\Big(\bE_{\pitp}[r_t+\gamma \cV_\lambda^\pitp(s_{t+1})|s_t,a_t]-\lambda\log\pit(a_t|s_t)-\cV_\lambda^\pitp(s_t)\Big)\Big|s_0\Big].
\end{multline*}
By definition, we have $\bE_\pitp[r_t+\gamma \cV_\lambda^\pitp(s_{t+1})|\mathcal{F}_t] = \cQ_\lambda^\pitp(s_t, a_t).$
Then,
\begin{align*}
    \cV_\lambda^{\pit}(s_0)-\cV_\lambda^{\pitp}(s_0) &= \bE_\pit\Big[\sum_{t=0}^\infty\gamma^t\Big(\cQ_\lambda^\pitp(s_t,a_t)-\lambda\log\pit(a_t|s_t)-\cV_\lambda^\pitp(s_t)\Big)\Big|s_0\Big],\\
    &= \bE_\pit\Big[\sum_{t=0}^\infty\gamma^t\Big(A_\lambda^\pitp(s_t,a_t)+\lambda\log\pitp(a_t|s_t)-\lambda\log\pit(a_t|s_t)\Big)\Big|s_0\Big],
\end{align*}
since $A_\lambda^\pitp(s,a) = \cQ_\lambda^\pitp(s,a)-\cV_\lambda^\pitp(s)-\lambda\log\pitp(a|s)$ for any $s,a$. Hence,
\begin{align*}
    \cV_\lambda^{\pit}(s_0)-\cV_\lambda^{\pitp}(s_0) &= \sum_{s, a}\sum_{t=0}^\infty \gamma^t\bP_\pitp(s_t=s|s_0)\pit(a|s)\Big(A_\lambda^\pitp(s,a)+\lambda\log\frac{\pitp(a|s)}{\pit(a|s)}\Big),\\
    &= \frac{1}{1-\gamma}\sum_{s,a}d_{s_0}^\pit(s)\pit(a|s)\Big(A_\lambda^\pitp(s,a)+\lambda\log\frac{\pitp(a|s)}{\pit(a|s)}\Big),
\end{align*}
which concludes the proof.
\end{proof}

The following lemma is a direct consequence of \eqref{eqn:grad-log}, and it was proposed in \cite{agarwal2020optimality}.
\begin{lemma}[Smoothness $\log\pi_\theta(a|s)$]
 For any $(s,a)\in\cS\times\cA$, $\log\pi_\theta(a|s)$ is smooth:
 
 \begin{equation}
     \|\nabla_\theta\log\pi_\theta(a|s)-\nabla_\theta\log\pi_{\theta^\prime}(a|s)\|_2 \leq \|\theta-\theta^\prime\|_2,
 \end{equation}
 
 for any $\theta,\theta^\prime\in\bR^d$.
 \label{lemma:smoothness}
\end{lemma}

\begin{proof}[Proof of Lemma \ref{lemma:lyapunov}]
By using the definition of $\Phi$:
\begin{align}
    \Phi(\pi_{t+1})-\Phi(\pi_t) &= \sum_{s,a}d_\mu^{\pi^*}(s)\pi^*(a|s)\log\frac{\pi_t(a|s)}{\pi_{t+1}(a|s)},\\
    &\leq -\eta_t\sum_{s,a}\dopt \nabla_\theta^\top\log\pi_t(a|s)g_t + \frac{\eta^2\|g_t\|_2^2}{2},
\end{align}
where the inequality follows from the smoothness of $\log\pi_\theta(a|s)$ shown in Lemma \ref{lemma:smoothness} \cite{agarwal2020optimality}. Then, by adding and subtracting the advantage function $A_\lambda^{\pi_t}(s,a)$ into the sum on the RHS of the above inequality, and using Lemma \ref{lemma:pdl}, we obtain the result.
\end{proof}

\subsection{Convergence Analysis of Entropy-Regularized NPG with Averaging}\label{subsec:q-npg}
\begin{proof}[Proof of Theorem \ref{thm:q-npg}]
We can write the Lyapunov drift in Lemma \ref{lemma:lyapunov} as follows:
\begin{multline}
    \Phi_{t+1}-\Phi_{t} \leq -\eta_t\lambda\Phi_t - \eta_t(1-\gamma)\Delta_t \\-\eta_t\sum_{s,a}d_\mu^{\pi^*}(s)\big(\pi^*(a|s)-\pi_t(a|s)\big)\Big[\phi_{s,a}^\top g_t - \cQ_\lambda^{\pi_t}(s,a)+\lambda\theta_t^\top\phi_{s,a}\Big] + \frac{1}{2}\eta_t^2\|g_t\|_2^2,
\end{multline}
where $\Phi_t:=\Phi(\pi_t)$ and $\Delta_t = \cV_\lambda^{\pi^*}(\mu)-\cV_\lambda^{\pi_t}(\mu).$
By Lemma \ref{lemma:q-npg}, we have $\|g_t\|_2 \leq 2R,$ for all $t \geq 1$. Thus, since $g_t = w_t-\lambda\theta_t$, we have:
\begin{multline}
    \Phi_{t+1}-\Phi_{t} \leq -\eta_t\lambda\Phi_t - \eta_t(1-\gamma)\Delta_t \\+\eta_t\sum_{s,a}d_\mu^{\pi^*}(s)\big(\pi_t(a|s)-\pi^*(a|s)\big)\Big[\phi_{s,a}^\top w_t - \cQ_\lambda^{\pi_t}(s,a)\Big] + 2\eta_t^2R^2.
    \label{eqn:ly-drift-a}
\end{multline}
Adding the baseline $b(s) = \bE_{a^\prime\sim\pi_t(\cdot|s)}[-\phi_{s,a^\prime}^\top w_t + \cQ_\lambda^{\pi_t}(s,a^\prime)]$ to the third summand (for each $s\in\cS$) on the RHS of the above does not change the inequality: 
\begin{multline}
    \Phi_{t+1}-\Phi_{t} \leq -\eta_t\lambda\Phi_t - \eta_t(1-\gamma)\Delta_t \\+\eta_t\underbrace{\sum_{s,a}d_\mu^{\pi^*}(s)\big(\pi_t(a|s)-\pi^*(a|s)\big)\Big[\nabla\log\pi_t(a|s) w_t - \Xi_\lambda^{\pi_t}(s,a)\Big]}_{(\blacksquare)} + 2\eta_t^2R^2.
    \label{eqn:ly-drift-b}
\end{multline}
Note that Lemma \ref{lemma:q-npg} implies $\pi^*(\cdot|s)\ll\pi_t(\cdot|s)$ for every $s\in\cS$. Thus, 
\begin{align*}
    &\sum_{s,a}d_\mu^{\pi^*}(s)\big(\pi_t(a|s)-\pi^*(a|s)\big)\Big(\nabla\log\pi_t(a|s) w_t - \Xi_\lambda^{\pi_t}(s,a)\Big)\\
    &\overset{(\spadesuit)}{\leq} \sum_{s,a}d_\mu^{\pi^*}(s)\frac{|\pi_t(a|s)-\pi^*(a|s)|}{\pi_t(a|s)}\pi_t(a|s)\Big[\nabla\log\pi_t(a|s) w_t - \Xi_\lambda^{\pi_t}(s,a)\Big],\\
    &\overset{(\clubsuit)}{\leq} \sqrt{\sum_sd_\mu^{\pi^*}(s)\chi^2(\pi^*(\cdot|s)\|\pi_t(\cdot|s))}\cdot\sqrt{\sum_sd_\mu^{\pi^*}(s)\pi_t(a|s)|\nabla\log\pi_t(a|s) w_t - \Xi_\lambda^{\pi_t}(s,a)|^2},\\
    &\overset{(\diamondsuit)}{\leq} \sqrt{\sum_sd_\mu^{\pi^*}(s)\chi^2(\pi^*(\cdot|s)\|\pi_t(\cdot|s))} \cdot \sqrt{\frac{\epsilon(R)}{1-\gamma}\|d_\mu^{\pi^*}/\mu\|_\infty},
\end{align*}
where $(\spadesuit)$ follows from the triangle inequality, $(\clubsuit)$ follows from H\"older's inequality, and $(\diamondsuit)$ follows from Assumption \ref{assumption:conc-coeff}. By using Lemma \ref{lemma:q-npg},
\begin{align*}
    \chi^2(\pi^*(\cdot|s)\|\pi_t(\cdot|s))=\sum_{a}\frac{(\pi_t(a|s)-\pi^*(a|s))^2}{\pi_t(a|s)}&\leq \sum_a\frac{\pi_t^2(a|s)+(\pi^*(a|s))^2}{p_{min}},\\
    &\leq \sum_a\frac{\pi_t(a|s)+\pi^*(a|s)}{p_{min}}=\frac{2}{p_{\min}},
\end{align*}
for any $s\in\cS$. Substituting these inequalities into \eqref{eqn:ly-drift-b} and noting that $\eta_t=\frac{1}{\lambda(t+1)}$, we obtain
\begin{align*}
    \Phi_{t+1}\leq \frac{t}{t+1}\Phi_t-\eta_t(1-\gamma)\Delta_t+\eta_t\sqrt{\frac{2\epsilon(R)\|d_\mu^{\pi^*}/\mu\|_\infty}{(1-\gamma)p_{min}}}+2\eta_t^2R^2,
\end{align*}
for every $t\geq 0$.
By induction,
\begin{equation}
    \Phi_T \leq - \frac{(1-\gamma)}{\lambda T}\sum_{t<T}\Delta_t + \frac{1}{\lambda}\sqrt{\frac{2\epsilon(R)\|d_\mu^{\pi^*}/\mu\|_\infty}{(1-\gamma)p_{min}}} + \frac{2R^2(1+\log T)}{\lambda^2T}.
\end{equation}
Since $\min_{t<T}\Delta_t \leq \frac{1}{T}\sum_{t<T}\Delta_t$, the proof follows.
\end{proof}

\rev{
\begin{proof}[Proof of Prop. \ref{prop:donsker-varadhan}]
    We will use the Lyapunov drift inequality given in \eqref{eqn:ly-drift-a}. Let $\mathsf{err}_t(s,a):=|w^\top\phi_{s,a}-\cQ_\lambda^{\pi_t}(s,a)|$. Then, by the Donsker-Varadhan variational representation (Theorem 3.16, \cite{hellstrom2023generalization}), we have
    \begin{align*}
        \sum_{s,a} d_\mu^{\pi^*}(s)\pi_t(a|s)\mathsf{err}_t(s,a) &\leq \log M_t(w_t)+\mathcal{D}_{\mathsf{KL}}(d_\mu^{\pi^*}\circ\pi^*\|d_\mu^{\pi_t}\circ\pi_t),\\
        \sum_{s,a} d_\mu^{\pi^*}(s)\pi_t(a|s)\mathsf{err}_t(s,a) &\leq \log M_t(w_t)+\mathcal{D}_{\mathsf{KL}}(d_\mu^{\pi^*}\circ\pi_t\|d_\mu^{\pi_t}\circ\pi_t),
    \end{align*}
    where $M_t(w_t)$ is the MGF of the error $\mathsf{err}_t(s,a)$ under $d_\mu^{\pi_t}\circ\pi_t$ at $1$, which is always bounded by $\|{\mathsf{err}}_t(\cdot,\cdot)\|_{\infty}$, which is also bounded by $O(R+\lambda)$ almost surely. Furthermore, by Lemma \ref{lemma:q-npg},
    $$\mathcal{D}_{\mathsf{KL}}(d_\mu^{\pi^*}\circ\pi_t\|d_\mu^{\pi_t}\circ\pi_t) \leq \mathcal{D}_{\mathsf{KL}}(d_\mu^{\pi^*}\circ\pi^*\|d_\mu^{\pi_t}\circ\pi_t)\leq \log\Big(\frac{\|d_\mu^{\pi^*}/\mu\|_\infty}{(1-\gamma)p_{min}}\Big),$$ which further implies that 
    \begin{equation*}
        \sum_{s,a}d_\mu^{\pi^*}(s)\big(\pi_t(a|s)-\pi^*(a|s)\big)\Big[\phi_{s,a}^\top w_t - \cQ_\lambda^{\pi_t}(s,a)\Big]\leq 2\log\Big(\frac{\|d_\mu^{\pi^*}/\mu\|_\infty M_t(w_t)}{(1-\gamma)p_{min}}\Big).
    \end{equation*}
    Substituting the above inequality into \eqref{eqn:ly-drift-a} and following the inductive steps for $\Phi_t,~t=0,1,\ldots,T-1$ as in the proof of Theorem \ref{thm:q-npg}, the proof is concluded.
\end{proof}
}

\rev{
\begin{proof}[Proof of Prop. \ref{prop:er-qnpg}]
    Let $\overline{\mathsf{err}}_t(s,a):=w^\top\nabla\log\pi_t(a|s)-\Xi_\lambda^{\pi_t}(s,a).$ Then, we have 
    \begin{align*}
        \sum_{s,a}d_\mu^{\pi^*}(s)\pi^*(a|s)|\overline{\mathsf{err}}_t(s,a)| &\leq \sqrt{\sum_{s,a}d_\mu^{\pi^*}(s)\pi^*(a|s)|\overline{\mathsf{err}}_t(s,a)|^2}, \\
        &\leq \sqrt{|\cA|\sum_{s,a}\breve{d}^\star(s,a)|\overline{\mathsf{err}}_t(s,a)|^2},
    \end{align*}
    since $\overbrace{d_\mu^{\pi^*}(s)\frac{1}{|\cA|}}^{\breve{d}^\star(s,a)}\pi^*(a|s)|\cA|\leq |\cA|\breve{d}^\star(s,a)$. Similarly, $$\sum_{s,a}d_\mu^{\pi^*}(s)\pi_t(a|s)|\overline{\mathsf{err}}_t(s,a)|\leq \sqrt{|\cA|\sum_{s,a}\breve{d}^\star(s,a)|\overline{\mathsf{err}}_t(s,a)|^2}.$$
    Using these two bounds in $(\blacksquare)$ in \eqref{eqn:ly-drift-b} for change-of-measure, we obtain
    \begin{align*}
        (\blacksquare) \leq 2\sqrt{|\cA|\sum_{s,a}\breve{d}^\star(s,a)|\overline{\mathsf{err}}_t(s,a)|^2}
        &\leq 2\sqrt{|\cA|\cdot\|\breve{d}^\star/\breve{d}_\nu^{\pi_t}\|_\infty\sum_{s,a}\breve{d}_\nu^{\pi_t}(s,a)|\overline{\mathsf{err}}_t(s,a)|^2},\\
        &\leq 2\sqrt{\frac{|\cA|\cdot\|\breve{d}^\star/\nu\|_\infty}{1-\gamma}\sum_{s,a}\breve{d}_\nu^{\pi_t}(s,a)|\overline{\mathsf{err}}_t(s,a)|^2},
    \end{align*}
    where the last inequality is due to $\breve{d}_\nu^{\pi_t}(s,a)\geq (1-\gamma)\nu(s,a),\forall s,a$. By substituting the above bound for $(\blacksquare)$ into \eqref{eqn:ly-drift-b} and following identical steps, we prove the result.
\end{proof}

}

\subsection{Convergence Analysis of Entropy-Regularized NPG}
\begin{proof}[Proof of Theorem \ref{thm:main}]  In the following, we will bound the terms in \eqref{eqn:drift} for step-sizes $\eta_t = \eta$ and $g_t = w_t$. First, note that
\begin{align}
    \nonumber -\eta\sum_{s,a}&\dopt \Big(\nabla_\theta^\top\log\pi_t(a|s)w_t-\big(\cQ_\lambda^{\pi_t}(s,a)-\lambda\log\pi_t(a|s)\big)\Big) \\ \nonumber  &
    \leq \eta \sqrt{\sum_{s,a}\dopt\Big(\nabla_\theta^\top\log\pi_t(a|s)w_t-\big(\cQ_\lambda^{\pi_t}(s,a)-\lambda\log\pi_t(a|s)\big)\Big)^2},\\
    & \label{eqn:drift-a} \leq \eta\sqrt{C_t L(w_t,\theta_t)} \leq \eta\sqrt{C^\star\ea},
\end{align}
where the second inequality holds by Cauchy-Schwarz inequality and Assumption \ref{assumption:dist-mismatch}, which implies $C_t \leq C^\star <\infty$, and the last inequality follows from the compatible function approximation error.

By the entropy-regularized NPG update, $w_t = (1-\gamma)\Big[G^{\pi_t}(\mu)\Big]^{-1}\nabla_\theta \cV_\lambda^{\pi_t}(\mu).$ By using this in \eqref{eqn:grad-log}, we have $\|\nabla_\theta\log\pi_t(a|s)\|_2 \leq 2$ by triangle inequality. From Proposition \ref{thm:policy-gradient}, we conclude that:
\begin{equation*}
    \|\nabla_\theta \cV_\lambda^{\pi_t}(\mu)\|_2 \leq \frac{2}{1-\gamma}\bE_{s\sim d_\mu^{\pi_t},a\sim\pi_t(\cdot|s)}[\cQ_\lambda^{\pi_t}(s,a)-\lambda\log\pi_t(a|s)],
\end{equation*}
since $\cQ_\lambda^{\pi_t}(s,a)-\lambda\log\pi_t(a|s)\geq 0$. This implies $\|\nabla_\theta \cV_\lambda^{\pi_t}(\mu)\|_2 \leq \frac{2}{1-\gamma}\sum_{s}d_\mu^{\pi_t}(s)\cV_\lambda^{\pi_t}(s).$ By Lemma \ref{lemma:prob-bounds}, we have $\|[G^{\pi_t}(\mu)]^{-1}\|_2 \leq 1/\sigma$. Using these two results with Cauchy-Schwarz inequality, we obtain:
\begin{equation}
    \|w_t\|_2 \leq \frac{2}{\sigma}\sum_{s}d_\mu^{\pi_t}(s)\cV_\lambda^{\pi_t}(s).
    \label{eqn:drift-b}
\end{equation}
By substituting \eqref{eqn:drift-a} and \eqref{eqn:drift-b} into \eqref{eqn:drift}, we have the following inequality:
\begin{multline}
\Phi(\pi_{t+1})-\Phi(\pi_t) \leq -\eta\lambda\Phi(\pi_t) - \eta(1-\gamma)\Big(\cV_\lambda^{\pi^*}(\mu)-\cV_\lambda^{\pi_t}(\mu)\Big)+\eta\sqrt{C^\star\ea}
\\ -\eta\sum_{s,a}\dopt \cV_\lambda^{\pi_t}(s) + \frac{2\eta^2}{\sigma^2}\Big(\sum_{s}d_\mu^{\pi_t}(s)\cV_\lambda^{\pi_t}(s)\Big)^2,
\end{multline}
where the step-size $\eta$ is chosen to make the summation of the last two terms on the RHS negative by using the bounds on $\cV_\lambda^\pi(s)$ provided in \eqref{eqn:val-bounds}. Therefore, we have the following Lyapunov drift inequality:
\begin{equation}
    \Phi(\pi_{t+1})-\Phi(\pi_t) \leq -\eta\lambda\Phi(\pi_t) - \eta(1-\gamma)\Big(\cV_\lambda^{\pi^*}(\mu)-\cV_\lambda^{\pi_t}(\mu)\Big)+\eta\sqrt{C^\star\ea}.
    \label{eqn:drift-final}
\end{equation}

We will use \eqref{eqn:drift-final} in two ways to obtain the results in Theorem \ref{thm:main}. Letting $$\Delta_t = \cV_\lambda^{\pi^*}(\mu)-\cV_\lambda^{\pi_t}(\mu),$$ note that
\begin{equation}
    \Phi(\pi_{t+1})\leq (1-\eta\lambda)\Phi(\pi_t) - \eta(1-\gamma)\Delta_t +\eta\sqrt{C^\star\ea}.
\end{equation}

By induction and noting that $\Phi(\pi_0) \leq \log|\cA|$, we obtain: \begin{equation}\Phi(\pi_{t+1}) \leq (1-\eta\lambda)^{t+1}\log|\cA| - \eta(1-\gamma)\sum_{k=0}^t(1-\lambda\eta)^{t-k}\Delta_k + \eta\sum_{k=0}^t(1-\eta\lambda)^{t-k}\sqrt{\ea{C^\star}},
\label{eqn:contraction}
\end{equation}
for any $t\geq 0$.

Using the bound on $\Phi(\pi_T)$ and rearranging the terms in the Lyapunov drift inequality \eqref{eqn:drift-final}, we bound the optimality gap for the last iterate:
$$\cV_\lambda^{\pi^*}(\mu) - \cV_\lambda^{\pi_T}(\mu) \leq \big(1-\lambda\eta\big)^T\frac{\log|\cA|}{\eta(1-\gamma)} + \frac{\sqrt{\ea C^\star}}{\lambda\eta(1-\gamma)}.$$ Using \eqref{eqn:contraction},
\begin{align}\min_{0\leq k \leq t}\Delta_k &\leq \frac{1}{\sum_{k=0}^t(1-\eta\lambda)^{t-k}}\sum_{k=0}^t(1-\eta\lambda)^{t-k}\Delta_{k},\\
&\leq \frac{(1-\eta\lambda)^{t+1}}{\sum_{k=0}^t(1-\eta\lambda)^k\eta(1-\gamma)}\log|\cA|+\frac{\sqrt{C^\star\ea}}{1-\gamma},
\end{align}

which concludes the proof.
\end{proof}
In the following, we provide a proof sketch for Lemma \ref{lemma:prob-bounds}.

\begin{proof}[Proof sketch for Lemma \ref{lemma:prob-bounds}] The proof consists of three steps.

\textbf{Step 1:} For any policy $\pi\in\Pi$ with $\min_{a\in\cA, s\in supp(\mu)}\pi(a|s) \geq p$ for $p > 0$, we can show that $\sigma_1(G^\pi(\mu)) \geq \sigma(p) > 0$ for some $\sigma(p) > 0$ under Assumption \ref{assumption:basis}. The proof follows from noting that $$u^\top F(\mu) u = \sum_s\mu(s)\text{Var}_{a\sim \mathsf{Unif}(\cA)}(\phi_{s,a}^\top u),$$ for any $u\in\bR^d$. Therefore, for any $u\in\bR^d$ and $s\in\cS$ such that $$\mu(s)\text{Var}_{a\sim \mathsf{Unif}(\cA)}(\phi_{s,a}^\top u) > 0,$$ we have $\mu(s)\text{Var}_{a\sim \pi(\cdot|s)}(\phi_{s,a}^\top u) > 0$ since $\min_{a\in\cA}\pi(a|s) > 0$.

\textbf{Step 2:} In the second step, we show that $\Phi(\pi)\leq \varepsilon$ for $\varepsilon > 0$ implies the following: \begin{align*}\pi(a|s) \geq \exp\Big(-\frac{\varepsilon}{\delta_\mu^*(s,a)}-\frac{h(\pi^*(\cdot|s))}{\pi^*(a|s)}\Big) = p^*(s,a,\varepsilon) > 0,\end{align*} for any $s\in supp(d_\mu^{\pi^*})$ where $\delta_\mu^{*}$ is the state-action visitation distribution under $\pi^*$. This bound directly follows from the definition of the potential function $\Phi$ (see Definition \ref{def:pf}).

\textbf{Step 3:} Let $p = \min_{s\in\mu(s),a\in\cA} p^*(s,a,\varepsilon)$ for $\varepsilon = \log|\cA| + \frac{\sqrt{C^\star\cdot \ea}}{\lambda}.$ For any policy $\pi\in\Pi$ with $\min_{s\in supp(\mu),a\in\cA}\pi(a|s) \geq p,$ we have shown in Step 1 that $\sigma_1(G^\pi(\mu)) \geq \sigma(p) = \sigma$. Let the step-size be $\eta \leq \min\{\sigma^2\eta_0, \frac{1}{2\lambda}\}$ with $\eta_0 = \frac{(1-\gamma)r_{min}}{(r_{max}+\lambda|\cA|)^2},$ and let $\tau = \inf\Big\{t\geq 1:\eta > \min\Big\{\Big(\sigma_1(G^{\pi_t}(\mu)\Big)^2\eta_0, \frac{1}{2\lambda}\Big\}\Big\}.$ Then, the inequality \eqref{eqn:contraction} holds for any $t < \tau$. Lemma \ref{lemma:prob-bounds} holds if and only if $\tau = \infty$. Suppose to the contrary that $\tau < \infty$. Hence, 
\begin{align*}
    \Phi(\pi_\tau) &\leq (1-\lambda\eta)^\tau\log|\cA|+\frac{\sqrt{C^\star\cdot\ea}}{\lambda}\leq \log|\cA|+\frac{\sqrt{C^\star\cdot\ea}}{\lambda},
\end{align*} which implies $\min\limits_{\substack{s\in supp(\mu)\\ a\in\cA}}\pi_{\tau}(a|s) \geq p,$ and therefore $\sigma_1(G^{\pi_\tau}(\mu)) \geq \sigma(p)$ by Step 1, which contradicts with the definition of $\tau$. This implies that $\tau = \infty$. Hence, the inequality in \eqref{eqn:contraction} holds and we have $\min_{s\in supp(\mu),a\in\cA}\pi_t(a|s) \geq p > 0$ and $\sigma_1(G^{\pi_t}(\mu)) \geq \sigma(p) = \sigma > 0$ for any $t\geq 1$, which concludes the proof.
\end{proof}

\subsection{Regularity of Random Features}\label{app:random-features}
In this section, we provide theoretical insights on the regularity of an important class of random features \cite{rahimi2007random} in the sense of Assumption \ref{assumption:basis}.

Consider a simple setting where $|\cA|=2$ and $\mu(s) = 1/|\cS|$ for all $s\in\cS$. The result can be extended to the general case by using similar arguments. For each $(s,a)\in\cS\times\cA$, $\phi_{s,a} \stackrel{iid}{\sim} \mathcal{N}(0,I_d/2)$. For this setting, we have Proposition \ref{prop:gaussian-regularity}, which implies that in the function approximation setting where $d \ll |\cS\times \cA|$, the ensemble of Gaussian random feature vectors satisfies the regularity condition in Assumption \ref{assumption:basis} with high probability.

\begin{proof}[Proof of Proposition \ref{prop:gaussian-regularity}]
Let $\Omega_d = \{z\in\bR^d:\|z\|_2 = 1\}$. For any $z\in\Omega_d$, we have: $$z^\top F(\mu) z = \frac{1}{4|\cS|}\sum_{s\in\cS}\Big([\phi_{s,1}-\phi_{s,2}]^\top z\Big)^2.$$ Let $\varphi_s \stackrel{iid}{\sim} \mathcal{N}(0,I_d)$. Then, for any $x>0$, 
\begin{align*}\sigma_1(F(\mu)) \geq \frac{x^2}{4}\min_{z\in\Omega_d}\frac{1}{|\cS|}\sum_{s\in\cS}\mathbbm{1}\{|\varphi_s^\top z| \geq x\}&\geq \frac{x^2}{4}\min_{z:\|z\|_2 \leq 1}\frac{1}{|\cS|}\sum_{s\in\cS}\mathbbm{1}\{|\varphi_s^\top z| \geq x\},\\
&= \frac{x^2}{4}\Big(1-\max_{z:\|z\|_2 \leq 1}\frac{1}{|\cS|}\sum_{s\in\cS}\mathbbm{1}\{|\varphi_s^\top z| \leq x\}\Big).
\end{align*}
By using the Rademacher complexity bound in \cite[Lemma 4]{satpathi2020role} to obtain and upper bound on $\max_{z:\|z\|_2 \leq 1}\frac{1}{|\cS|}\sum_{s\in\cS}\mathbbm{1}\{|\varphi_s^\top z| \leq x\}$, we conclude the proof.
\end{proof}

Another class of random features is the so-called neural tangent kernel (NTK) features, which attracted significant attention for the theoretical analysis of neural networks \cite{jacot2018neural, wang2019neural, liu2019neural}. Each state-action pair $(s,a)$ is represented by a vector $\varphi(s,a)\in\bR^{d^\prime}$. For a single-layer neural network of width $m > 1$, the NTK feature is defined as follows:
    $\phi_{s,a} = \left[\frac{1}{\sqrt{m}}c_i\cdot\varphi(s,a)\cdot\mathbbm{1}\{W_i^\top\varphi(s,a) \geq 0\}\right]_{i\in[m]},$
where $c_i\sim Rademacher$, $W_i\sim \mathcal{N}(0, I_{d^\prime})$ for $i\in\{1,2,\ldots,m\}$. In this case, $d = m\cdot d^\prime$. For randomly-generated $\{\varphi(s,a)\sim\mathcal{N}(0, I_{d^\prime}):(s,a)\in\cS\times\cA\}$, we present the minimum eigenvalue of $$F(\mu) = \bE_{s\sim\mu,a\sim \mathsf{Unif}(\cA)}[\nabla\log\pi_0(a|s)\nabla^\top\log\pi_0(a|s)],$$ in Figure \ref{fig:ntk}. Note that $\sigma_1(G^{\pi_0}(\mu)) \geq (1-\gamma)\sigma_1(F(\mu))$ since $d_\mu^{\pi_0}(s) \geq (1-\gamma)\mu(s)$ for all $s$.
\begin{figure}
    \centering
    \includegraphics[width = .6\columnwidth]{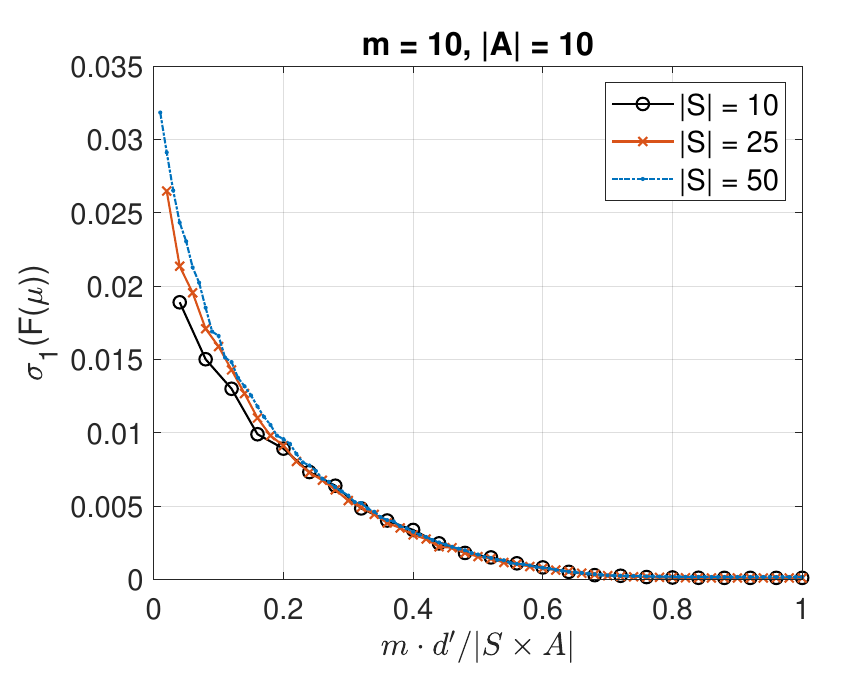}
    \caption{Minimum eigenvalue of $F(\mu)$ for neural tangent kernel features.}
    \label{fig:ntk}
\end{figure}
Figure \ref{fig:ntk} indicates that NTK features satisfy the regularity condition given in Assumption \ref{assumption:basis} for $d \ll |\cS\times\cA|$, i.e., in the function approximation regime.

\section{Proofs for Sample-Based Entropy-Regularized NPG}\label{app:actor-critic}
\subsection{Proof of Proposition \ref{prop:sample-based-npg}}
\begin{proof}
For any $t\leq T$ and $w\in\bR^d$, let
\begin{align}
    L_{0,t}(w) &= \bE_{s\sim d_\mu^{\pi_t},a\sim\pi_t(\cdot|s)}\big[\big(\nabla^\top\log\pi_t(a|s)w-\Xi_\lambda^{\pi_t}(s,a)\big)^2\big],\\
    \widehat{L}_{0,t}(w) &= \bE_{s\sim d_\mu^{\pi_t},a\sim\pi_t(\cdot|s)}\big[\big(\nabla^\top\log\pi_t(a|s)w-\widehat{\Xi}_\lambda^{\pi_t}(s,a)\big)^2\big].
\end{align}
Also, it is easy to see that $\bE_{(s,a)\sim d_\mu^{\pi_t}\circ\pi_t}\left[\left(\widehat{\Xi}_\lambda^{\pi_t}(s,a)-{\Xi}_\lambda^{\pi_t}(s,a)\right)^2\right] \leq \epsilon_{critic}.$
Thus, by using the inequality $(x+y)^2 \leq 2x^2 + 2y^2,~x,y\in\bR$, we have:
\begin{equation}\min_w~\widehat{L}_{0,t}(w) \leq \min_w~2L_{0,t}(w)+2\epsilon_{critic}.\label{eqn:error-analysis-a}\end{equation}
Similarly,
${L}_{0,t}(w_t) \leq 2L_{0,t}(w_t)+2\epsilon_{critic}.$
Thus, taking expectation over the samples, we obtain:
\begin{align*}
    \bE[L_{0,t}(w_t)] &\leq 2\bE[\widehat{L}_{0,t}(w_t)] + 2\epsilon_{critic} \leq 2\min_w\widehat{L}_{0,t}(w)+2\epsilon_{actor}+2\epsilon_{critic},\\
    &\leq 4\min_wL_{0,t}(w) + 2\epsilon_{actor} + 6\epsilon_{critic},
\end{align*}
where the second line follows from the definition of $\epsilon_{actor}$ and the last line follows from \eqref{eqn:error-analysis-a}. Since $\sup_{t\geq 0}\min_w~L_{0,t}(w)\leq \epsilon(R)$, the proof follows.
\end{proof}

\bibliographystyle{siamplain}
\bibliography{references}
\end{document}